\documentclass{article}

\usepackage{multicol,url}
\usepackage{amsmath,amsbsy,amssymb}
\usepackage{graphicx}
\usepackage[usenames,dvipsnames]{color}
\usepackage{fixltx2e}
\usepackage{textcomp}
\usepackage{subcaption}
\usepackage[normalem]{ulem}
\newtheorem {definition}   {Definition}

\newtheorem {lemma}       {Lemma}
\newtheorem {proposition} {Proposition}

\newtheorem {corollary}     {Corollary}
\newtheorem {note}          {Note}
\newtheorem {example}     {Example}

\newcommand {\tup}[1]      {{\langle #1 \rangle}}

\newcommand {\floor}[1]    {{\lfloor #1 \rfloor}}

\newcommand{\powerset}  {\wp}
\newcommand{\ABF}         {{\sf ABF}}

\newcommand{\F}             {{\sf F}}

\newcommand{\CNL}{{\sf Cn}_{\mathfrak L}}
\usepackage{tikz}
\usetikzlibrary{shapes,arrows,positioning,backgrounds,fit,calc,shapes.misc}
\newcommand{\Atoms}  {{\cal A}}
\usepackage{colortbl}

\newcommand{\qed}         {\hfill$\Box$}

\usepackage[switch,pagewise]{lineno}

\begin{document}
\title{Argumentative Characterizations of \\ (Extended) Disjunctive Logic Programs}

%

\author{
Jesse Heyninck$^{1,2}$ and Ofer Arieli$^{3}$ \medskip \\
{\small $^1$ Department of Computer Science, Open Universiteit, the Netherlands } \\
{\small $^2$ Department of Computer Science, University of Cape Town, South-Africa } \\
{\small $^3$ School of Computer Science, Tel-Aviv Academic College, Israel}}

\date{}
\maketitle              

\begin{abstract}
This paper continues an established line of research about the relations between argumentation theory, particularly
assumption-based argumentation, and different kinds of logic programs. In particular, we extend known result of 
Bondarenko, Dung, Kowalski and Toni, and of Caminada and Schulz,
by showing that assumption-based argumentation can represent not only normal logic programs, but also disjunctive logic programs
under the stable model semantics. For this, 
we consider some inference rules for disjunction that the core logic of the argumentation frameworks should respect, 
and show the correspondence to the handling of disjunctions in the heads of the logic programs' rules.

Underconsideration in Theory and Practice of Logic Programming (TPLP).
\end{abstract}

\section{Introduction}

Logic programming (LP) and formal argumentation are two primary disciplines involving knowledge representation and non-monotonic reasoning. 
Assumption-based argumentation (ABA, for short )~\cite{Bondarenko1997,CXST18} is a well-established branch in argumentation theory, aimed at providing 
coherent sets of formulas that admit other sets of formulas as their contraries, based on assumptions, contrariness operators, and rules in corresponding deductive 
systems. ABA was inspired by Dung’s semantics for abstract argumentation frameworks~\cite{Dung1995} and logic programming with its dialectical interpretation of 
the acceptability of negation-as-failure assumptions based on the `failure to prove the converse'~\cite{Ll87}. Thus, ABA can be viewed as an argumentative 
interpretation of LP semantics.

The similar ground of LP and ABA calls upon translation methods for revealing the exact relations between them, and for importing 
reasoning methods from one formalism to the other. For example, argumentative characterizations of logic programming have been 
proven useful for explanation~\cite{schulz2015characterising,schulz_toni_2016}, visualization of inferences in logic programming~\cite{schulz2015graphical},  
and debugging of logic programs~\cite{thevapalan2021establish}. Among the works that relate LP and ABA we recall the one of Bondarenko 
et~al.~\cite{Bondarenko1997}, who were the first to show a correspondence between stable models (respectively, the well-founded model) 
of a normal logic program and stable extensions (respectively, the grounded extension) of the associated ABA framework (see Theorem~3.13, respectively
Theorem~6.3 by~\cite{Bondarenko1997}), 
and the work of 
Caminada and Schulz~\cite{caminada2017equivalence,CS18}
that provides a one-to-one correspondence between the 3-valued stable models~\cite{Prz90} (respectively, the regular models~\cite{YY94})
for normal logic programs, and complete labelings (respectively, preferred labelings) for ABA frameworks. 
A recent work~\cite{Wa20} shows that answer sets of a (non-disjunctive) extended logic programs 
can be captured by stable extensions of the translated ABA frameworks. These works were restricted to logic programs, where only atoms 
or their (classical) negations are allowed in the head of the rules. Yet, a faithful modeling of real-world problems often requires to cope with 
\emph{incomplete information\/}, which is not possible in the scope of such logic programs. This is a primary motivation in the introduction of  
\emph{disjunctive logic programs\/}, where (classical) disjunctions are allowed in the heads of the rules and negations 
(sometimes called `negation-as failure',~\cite{gelfond1991classical,Ll87}) may occur in the rules' bodies. Reasoning with uncertainty is also a principle
motivation behind \emph{extended disjunctive logic programs\/}, where a classical negation is also permitted, both in the rules bodies and their heads. Indeed,
disjunctive logic programs have been efficiently implemented and widely applied, and so become a key technology in knowledge representation 
(see, e.g.,~\cite{su2015extensions}), although under the usual complexity assumptions, they were shown to be strictly more expressive than normal logic
programs~\cite{eiter1997disjunctive,gottlob1994complexity}.

Despite the equivalence between ABA semantics and the semantics of normal logic programs that has already been obtained in a number of 
works~\cite{caminada2017equivalence,CS18}, it is not obvious that such a correspondence carries on to disjunctive logic programs.
In fact, there are some a-priori indications to the contrary, at least in some fundamental cases. For instance, the correspondence shown in~\cite{caminada2017equivalence} 
between 3-valued stable models of normal logic programs and complete extensions of ABA, breaks down when disjunctions may appear in the rules' 
heads. This is simply due to the fact that there are disjunctive logic programs without 3-valued stable models~\cite{Pr91}, while (flat) ABA frameworks 
always have complete extensions~\cite{CXST18}. Therefore, in this paper we set out to generalize  the argumentative characterization of logic programming 
for disjunctive logic programs and their extended variations. Naturally, this raises the primary research question of this paper, namely:
{\em ``Can existing translations from normal logic programming into assumption-based argumentation be extended, in some `natural way',  to 
disjunctive logic programs?''\/}

This question is answered affirmatively by showing that,
at least as far as two-valued stable models are involved, disjunctive logics programs (and even extended disjunctive logic 
programs, allowing also strong negation in the rules) {\em can be\/} faithfully represented in terms of assumption-based argumentation frameworks.
For this, we incorporate the ideas in~\cite{AH21,AH25,HA20}, generalizing standard ABA frameworks to propositional formulas (as the defeasible or strict 
assumption at hand), expressed and evaluated in propositional (Tarskian) logics. This allows to augment the underlying core logic of the ABA framework (based only on Modus Ponens,
MP) with the inference rules Resolution (Res)  and Reasoning by Cases (RBC), for handling disjunctive assertions, and so associate the stable extensions of such frameworks
with the (2-valued) stable models of the corresponding disjunctive logic programs. On the other hand, we show that even when 3-valued stable models do exist, they might 
not correspond to complete models in the translated argumentation theory. 

\paragraph{The structure of the paper}
In the next section we review some basic notions behind assumption-based argumentation and 
disjunctive logic programming. In Section~\ref{sec:DLP2ABA}, which is the main part of the paper, we show how (the stable models of) the latter 
can be represented by (the stable extensions of) the former. The converse is discussed in Section~\ref{sec:ABA2DLP}. In Section~\ref{sec:other-semantics}
we provide some negative results, showing that this correspondence does not carry on to models and extensions that are not necessarily two-valued and stable. 
On the other hand, as shown in Section~\ref{sec:EDLP}, our results are easily generalized to extended disjunctive logic programs, where a classical negation 
is also allowed, in addition to the negation-as-failure connective. In Section~\ref{sec:conclusion} (referring also to the appendix) we discuss some related work and conclude.

\paragraph{Relations with previous work}
This paper is an updated and extended version of the paper in~\cite{HA19}, where we first revealed the relations between ABA and disjunctive LP. 
In particular, we give here full proofs for the main results of the paper (in~Section~\ref{sec:correctness-proofs}), including proofs that were omitted in the conference 
paper, and revisions of other proofs in~\cite{HA19},  add further illustrations to the main concepts, consider also 3-valued semantics (and corresponding models) for LP, 
discuss extended logic programs, and refer to related work in more detail. Our work in~\cite{HA19}, which is extended here, closes a gap in the investigations 
of the connections between LP and formal argumentation, mainly in the presence of uncertain information. For this purpose, we extend the logical setting of the work in~\cite{Bondarenko1997,caminada2017equivalence,CS18}, consisting of Modus Ponens (MP) as the sole inference rule, by two additional inference rules, (Res) and (RBC)
mentioned previously, which also allow to handle disjunctive information. This extended logical setting (with some adjustments) serves as a basis for a number 
of other recent follow-up papers, e.g.\ the ones by Wakaki~\cite{Wa22,Wa24}, discussing disjunctive information in ABA frameworks in relation to the way it is evaluated 
by semantics for extended disjunctive logic programs. The work in~\cite{Wa24} shows, furthermore, that (a variation\footnote{See Section~\ref{sec:conclusion} for 
some further details on this variation.} of) our logical setting is also useful for linking ABA frameworks and a number of other formalisms 
for non-monotonic reasoning, including disjunctive default logic~\cite{EC96,GPLT91}, prioritized circumscription~\cite{Li85}, and possible 
model semantics for extended disjunctive logic programs~\cite{SI00}. 
The motivation behind our work is therefore twofold: in the theoretical level it extends, by means of a simple translation, the known links between ABA and LP, 
and on the more practical side it vindicates the usefulness of argumentative and LP-based frameworks  in handling not only inconsistent information, but also 
incomplete one.\footnote{A similar motivation stands behind the proposal of Gelfond et al.\ in~\cite{GPLT91} to extend Reiter’s default logic~\cite{reiter1980logic} 
to disjunctive default logic in order to overcome some problems of the former in handling disjunctive information.}

\section{Preliminaries}

We start with a brief review of the main concepts that are related to assumption-based argumentation (Section~\ref{sec:ABA}) and disjunctive logic programs
(Section~\ref{sec:DLP}).

\subsection{Assumption-Based Argumentation}
\label{sec:ABA}

We denote by ${\mathcal L}$ a propositional language. We shall assume that ${\mathcal L}$ contains a conjunction (denoted as usual in logic programming by a comma), 
disjunction $\vee$, implication $\rightarrow$, a negation operator $\sim$, and propositional constants {\sf F}, {\sf T} for falsity and truth.
Atomic formulas in ${\mathcal L}$ are denoted by $p,q,r$ (possibly indexed), literals (i.e., atomic formulas or their negation) are denoted by $l$ (possibly indexed),
compound formulas are denoted by  $\psi,\phi,\sigma$, and sets of formulas in ${\mathcal L}$ are denoted by $\Gamma,\Delta,\Theta,\Lambda$. 
When considering extended disjunctive logic programs, we shall assume that the language also contains another kind of negation, denoted $\neg$ (the exact meaning 
of each connective will be defined in the sequel). In what follows, we denote by $\sim\!\Gamma$ the set $\{\sim\!\gamma \mid \gamma\in\Gamma\}$.
The powerset of $\Gamma$ is denoted $\powerset(\Gamma)$.

\begin{definition}
\label{def:logic}
{\rm A (propositional) {\em logic\/} for a language ${\mathcal L}$ is a pair ${\mathfrak L} = \tup{{\mathcal L},\vdash}$, where $\vdash$ 
is a (Tarskian) consequence relation for ${\mathcal L}$, that is, a binary relation between sets of formulas and formulas in ${\mathcal L}$, 
satisfying the following properties:
\begin{itemize}
   \item {\sl Reflexivity\/}: if $\psi \in \Gamma$ then $\Gamma \vdash \psi$.
   \item {\sl Monotonicity\/}: if $\Gamma \vdash \psi$ and $\Gamma \subseteq \Gamma^\prime$, then $\Gamma^\prime \vdash \psi$.
   \item {\sl Transitivity\/}: if $\Gamma \vdash \psi$ and $\Gamma^\prime,\psi \vdash \phi$, then $\Gamma,\Gamma^\prime \vdash \phi$.
\end{itemize}   
} \end{definition}

The next definition, adapted from~\cite{HA20}, generalizes the definition in~\cite{Bondarenko1997} of assumption-based argumentation frameworks.

\begin{definition}
\label{def:ABF}
{\rm An \emph{assumption-based framework\/} (ABF, for short) is a tuple $\ABF = \tup{\mathfrak{L}, \Gamma, \Lambda, -}$, where: 
\begin{itemize}
\item ${\mathfrak L} = \tup{{\mathcal L},\vdash}$ is a propositional logic.
\item $\Gamma$ (the \emph{strict assumptions\/}) and $\Lambda$ (the \emph{candidate or defeasible assumptions\/}) 
         are distinct countable sets of $\mathcal{L}$-formulas, where the former is assumed to be $\vdash$-consistent 
         (i,.e, $\Gamma \not\vdash \F$) and the latter is assumed to be nonempty. 
\item $- : \Lambda \rightarrow \powerset(\mathcal{L})$ is a contrariness operator, assigning a finite set of ${\mathcal L}$-formulas 
         to every defeasible assumption in $\Lambda$.
\end{itemize}
} \end{definition}

\begin{note} 
{\rm Unlike the setting of~\cite{Bondarenko1997}, an ABF may be based on {\em any\/} propositional logic $\mathfrak{L}$. 
Also, the strict as well as the candidate assumptions are formulas that may not be just atomic. Concerning the contrariness operator, 
note that it is not a connective of ${\mathcal L}$, as it is restricted only to the defeasible assumptions.
} \end{note}

Defeasible assertions in an ABF may be attacked by counterarguments. 

\begin{definition}
\label{def:non-proiritized-attack}
{\rm Let $\ABF=\tup{\mathfrak{L}, \Gamma, \Lambda,-}$ be an assumption-based framework, $\Delta,\Theta \subseteq \Lambda$, 
and $\psi \in \Lambda$. We say that $\Delta$ \emph{attacks\/} $\psi$ iff $\Gamma,\Delta \vdash \phi$ for some $\phi \in -\psi$. 
Accordingly, $\Delta$ attacks $\Theta$ if $\Delta$ attacks some $\psi \in \Theta$.
} \end{definition}

\begin{example}
\label{examp:1}
{\rm Let ${\sf ABF}=\tup{\mathfrak{L}_{\sf MP}, \{{\sim\!q} \rightarrow p\}, \{{\sim\!p}, \: {\sim\!q}\},-}$, where $-{\sim\!p} = \{p\}$ and 
$-{\sim\!q} = \{q\}$, and where $\mathfrak{L}_{\sf MP}$ consists of the sole inference rule Modus Ponens:
\[ \left. \begin{array}{lll}
    \mbox{[{\sf MP}]} \  \footnotemark \hspace*{2mm} & 
          \displaystyle \frac{ \phi_1,\ldots,\phi_n \rightarrow \psi \hspace{6mm}
                 \phi_1 \hspace{5mm} \phi_2 \hspace{3mm} \cdots \hspace{3mm} \phi_n}   {\psi}
 \end{array} \right. \]
\footnotetext{As usual, the formulas above the fragment line are the rule's conditions and the formula below the line is the rule's conclusion.}
This gives rise to the following visual representation of ${\sf ABF}$ in terms of an {\em attack diagram\/}:
\begin{center}
\begin{tikzpicture}
\tikzset{edge/.style = {->,> = latex'}}

\node (simp) at (-0.5,0) {$\{\sim\! p\}$};
\node (simq) at (3,0) {$\{\sim\! q\}$};
\node (simpsimq) at (1,1.5) {$\{{\sim\!p}, \: {\sim\!q}\}$};

\draw[edge, loop, looseness=4] (simpsimq) to (simpsimq);
\draw[edge] (simq) to (simp);
\draw[edge] (simq) to (simpsimq);
\end{tikzpicture}
\end{center}

This diagram may be viewed as a directed graph whose nodes are sets of defeasible assumptions and where a directed arrow represents an attack of 
the set at the origin of the arrow on the set at the arrow's end. Note that, by MP, it hold that $\sim\! q,\: {\sim\!q} \rightarrow p \vdash p$, and so in our case 
every set that contains ${\sim\!q}$ attacks any set that contains ${\sim\!p}$.
} \end{example}

The last definition gives rise to the following adaptation to ABFs of the usual 
semantics for abstract argumentation frameworks~\cite{Dung1995}.

\begin{definition}
\label{def:semantics} 
{\rm (\cite{Bondarenko1997})
Let $\ABF=\tup{\mathfrak{L}, \Gamma, \Lambda,-}$ be an assumption-based framework, and let $\Delta \subseteq \Lambda$. 
Then $\Delta$ is \emph{conflict-free\/} 
             iff there is no $\Delta'\subseteq\Delta$ that attacks some $\psi \in \Delta$.
We say that $\Delta$ is a \emph{stable extension\/} of $\ABF$ iff it is conflict-free, and attacks every $\psi \in \Lambda\setminus \Delta$. The set of stable extensions
of $\ABF$ is denoted by ${\sf Stb}(\ABF)$.\footnote{In many presentations of assumption-based argumentation, stable extensions are required to be {\em closed\/},
i.e., they should contain any assumption they imply. Since the translation below will always give rise to the so-called \emph{flat} ABFs (that is, ABFs for which 
a set of assumptions can never imply assumptions outside the set; See Note~\ref{note:flat-translation} below), closure of extensions is trivially satisfied in our case.}
} \end{definition}

\begin{example}
\label{examp:1a}
{\rm Consider again the assumption-based argumentation framework in Example~\ref{examp:1}. The sole stable extension of this framework is $\{{\sim\!q}\}$.
} \end{example}

\subsection{Disjunctive Logic Programs}
\label{sec:DLP}

\begin{definition}
\label{def:DLP} 
{\rm (\cite{Pr91}) A {\em disjunctive logic program\/} (DLP) $\pi$ is a finite set of rules of the form 
\[ (\star) \hspace*{15mm} q_1, \ldots, q_m, {\sim\!r_1}, \ldots, {\sim\!r_k} \rightarrow p_1 \lor \ldots \lor p_n  \] 
where $m,k \geq 0$ and $n\geq 1$.\footnote{The assumption that $n \geq 1$ means that in this work we do not consider \emph{constraints\/},
i.e., rules with empty heads (This is also the assumption of Caminada and Schulz in their transformation of normal logic programs to  assumption-based 
argumentation~\cite{caminada2017equivalence,CS18}, which this work extends). The incorporation of constraints can be dealt with by
interpreting constraints $q_1, \ldots, q_m, {\sim\!r_1}, \ldots, {\sim\!r_k} \rightarrow $ 
as rules of the form $q_1, \ldots, q_m, {\sim\!r_1,} \ldots, {\sim\!r_k} \rightarrow {\sf F}$. In turn, this requires to introduce {\em explosiveness assumptions\/},
expressing that from falsity any conclusion may be derived (see, e.g.,~\cite{Wa24}). We leave this extension for future work.} We say that
$p_1 \lor \ldots \lor p_n$ the {\em head\/} (conclusion) of the rule, and  that $q_1, \ldots, q_m, \sim\!r_1, \ldots, \sim\!r_k$
is the {\em body\/} (assumptions) of the rule. 
\begin{itemize}
     \item A logic program $\pi$ is {\em positive\/}, if $k=0$ for every rule in $\pi$ (i.e., the negation as failure operator $\sim$ does not appear in $\pi$). 
     \item When each head of a rule in $\pi$ is  either empty or consists of an atomic formula (i.e., $n \leq 1$), we say that $\pi$ is a {\em normal logic program\/}. 
\end{itemize}     
We denote by $\Atoms(\pi)$  the set of atomic formulas that appear in $\pi$.
} \end{definition}

Intuitively, the rule in~$(\star)$ indicates that if $q_i$ holds  for every $1 \leq i \leq m$ and $r_i$ is not provable 
for every $1 \leq i \leq k$ , then either of the $p_i$'s, for $1 \leq i \leq n$, should hold. In what follows, unless otherwise stated, 
when referring to a logic program we shall mean that it is disjunctive. The semantics of a logic program $\pi$ is defined as follows:

\begin{definition}
\label{def:satisfiability}
\label{def:models}
{\rm  Let $M$ be set of atomic formulas, $p,p_i,q_j$ atomic formulas, and $l_j$ literals. We denote: 
\begin{itemize}
     \item $M \models p$ iff $p \in M$,
     \item $M \models \:\sim\!p$ iff $p \not\in M$,
     \item $M \models p_1 \vee \ldots \vee p_n$ iff $M \models p_i$ for some $1 \leq i \leq n$,
     \item $M \models l_1, \ldots, l_m$ iff $M \models l_j$ for every $1 \leq j \leq m$,
     \item $M \models l_1, \ldots, l_m \rightarrow p_1 \vee \ldots \vee p_n$ iff either $M \models p_1 \vee \ldots \vee p_n$ or  $M \not\models  l_1, \ldots, l_m$ 
             (the latter means that it is {\em not the case\/} that $M \models l_1, \ldots, l_m$). 
\end{itemize}
Note that $M$ may be viewed as an interpretation into $\{t,f\}$, where $M(p) = t$ iff $p \in M$ (iff $M \models p$).
When $M \models \psi$, we say that $M$ {\em satisfies\/} $\psi$. Given a logic program $\pi$, we denote by $M \models \pi$ that $M$ satisfies every $\psi \in \pi$.
In that case, we say that $M$ is a {\em model\/} of $\pi$.
} \end{definition}

Thus, a model of a rule either falsifies at least one of the conjuncts in the rule's body, or validates at least one of the disjuncts in the rule's head.
A particular family of models for disjunctive logic programs called {\em stable\/} (see~\cite{Pr91}) is defined next. 

\begin{definition}
\label{def:stable-model}
{\rm Let $\pi$ be a disjunctive logic program and let  $M \subseteq \Atoms(\pi)$.
\begin{itemize}
    \item The {\em Gelfond-Lifschitz reduct\/}~\cite{GL88} of $\pi$ with respect to $M$ is the (positive) disjunctive logic program $\pi^M$, where  
             $q_1, \ldots, q_m\rightarrow p_1 \lor \ldots \lor p_n \in \pi^M$ iff there is a rule 
             $q_1, \ldots, q_m, \sim\!r_1, \ldots, \sim\!r_k\rightarrow p_1 \lor \ldots \lor p_n \in \pi$ and $r_i \not\in M$ for every
             $1\leqslant i\leqslant k$.
   \item $M$ is a {\em stable model\/} of $\pi$ iff it is a $\subseteq$-minimal model of $\pi^M$.
\end{itemize}
} \end{definition}

\begin{example}
\label{examp:2}
{\rm Consider the disjunctive logic program: $\pi_1=\{{\sim\!p} \rightarrow q \lor r \}$. Below are the different combinations of atoms in this case and their reducts. 
The two stable models of $\pi_1$ are marked by a gray background. 

\smallskip
\def\arraystretch{1.2}
\setlength{\tabcolsep}{4mm}
\begin{center}
\begin{tabular}{|l|l|c||l|l|c|}
\hline 
$i$& $M_i$ & $\pi_1^{M_i}$                                                                                                          & $i$& $M_i$ & $\pi_1^{M_i}$ \\ \hline                      
$1$ & $\emptyset$  & $\{ \rightarrow q \lor r \}$                                                                            & $5$ & $\{p, q\}$ & $\emptyset$ \\ \hline
$2$ & $\{p\}$ & $\emptyset$                                                                                                         & $6$ & $\{p, r\}$ & $\emptyset$ \\ \hline 
\cellcolor{gray!30}$3$ & \cellcolor{gray!30}$\{q\}$ & \cellcolor{gray!30} $\{ \rightarrow q \lor r \}$   & $7$ & $\{q, r\}$ & $\{ \rightarrow  q \lor r \}$ \\ \hline 
 \cellcolor{gray!30}$4$ & \cellcolor{gray!30}$\{r\}$ & \cellcolor{gray!30} $\{ \rightarrow q \lor r \} $  &  $8$ & $\{p, q, r\}$ & $\emptyset$ \\ \hline 
\end{tabular}
\end{center}
Thus, $\{q\}$ and $\{r\}$ are the (only) stable models of $\pi_1$.
} \end{example}

\section{Representation of DLP by ABA}
\label{sec:DLP2ABA}

Given a disjunctive logic program $\pi$, we show a one-to-one correspondence between the stable models of $\pi$ 
onto the stable extensions of an ABA framework that is induced from $\pi$. First, we describe the translation and then 
prove its correctness.

\subsection{The Translation}
\label{sec:translation}

All the ABA frameworks that are induced from disjunctive logic programs will be based on the same core logic, which is constructed by
the three inference rules Modus Ponens (${\sf MP}$), Resolution (${\sf Res}$) and Reasoning by Cases (${\sf RBC}$):
\[ \left. \begin{array}{lll}
 \mbox{[{\sf MP}]} \hspace*{2mm} & 
          \displaystyle \frac{ \phi_1,\ldots,\phi_n \rightarrow \psi  \hspace{6mm}
                 \phi_1 \hspace{5mm} \phi_2 \hspace{3mm} \cdots \hspace{3mm} \phi_n}   {\psi} \bigskip \\  
       \mbox{[{\sf Res}]} \hspace*{2mm} &
          \displaystyle \frac{\psi'_1 \vee \ldots \vee \psi'_m \vee \phi_1 \vee \ldots \vee \phi_n \vee  \psi''_1 \vee \ldots \vee \psi''_k \hspace{6mm}
                 \sim\!\phi_1 \hspace{3mm} \cdots \hspace{3mm} \sim\!\phi_n}   
                {\psi'_1 \vee \ldots \vee \psi'_m \vee \ldots \vee \psi''_1 \vee \ldots \vee \psi''_k} \bigskip \\  
 \mbox{[{\sf RBC}]} \hspace*{2mm} & 
          \displaystyle \frac{\begin{array}{c} \phi_1 \\ \vdots \\ \psi \end{array} \hspace*{7mm}
                                     \begin{array}{c} \phi_2 \\ \vdots \\ \psi \end{array} \hspace*{5mm}
                                     \begin{array}{c} \ \\ \ \\ \cdots \end{array} \hspace{5mm}
                                     \begin{array}{c} \phi_n \\ \vdots \\ \psi \end{array} \hspace*{7mm}
                                     \begin{array}{c} \ \\ \ \\  \phi_1 \vee \ldots \vee \phi_n \end{array}}{\psi}
 \end{array} \right. \]
In what follows we denote by ${\mathfrak L} = \tup{{\mathcal L},\vdash}$ the logic based on the language ${\mathcal L}$ which consists of
disjunctions of atoms ($p_1\lor\ldots \lor p_n$ for $n \geq 1$), negated atoms ($\sim\!p$), or formulas of the forms of the program rules in
Definition~\ref{def:DLP}. Accordingly, we shall use only fragments of the inference rules above, 
in which in {\sf [Res]} and {\sf [RBC]} the formulas $\psi,\psi_i$ are disjunctions of atomic formulas and $\phi_i$ are atomic formulas.
In {[\sf MP]}, $\psi$ is a disjunction of atomic formulas and $\phi_i \in \{p_i, {\sim\!p_i}\}$ are literals. Now, we denote $\Delta\vdash \phi$ iff $\phi$ 
is either in $\Delta$ or is derivable from $\Delta$ using the inference rules above. In other words, $\Delta \vdash \phi$ iff 
$\phi \in {\sf Cn}_{\mathfrak L}(\Delta)$, where ${\sf Cn}_{\mathfrak L}(\Delta)$ is the ${\mathfrak L}$-based transitive closure of 
$\Delta$ (namely, the $\subseteq$-smallest set that contains $\Delta$ and is closed under {\sf [MP]}, {\sf [Res]} and {\sf [RBC]}). 

\begin{note}
\label{note:for:flatness}
{\rm For any $\phi\in\CNL(\Delta)$, if $\phi$ is not of the form $p_1\lor\ldots\lor p_n$ then $\phi\in\Delta$.
}\end{note}

\begin{note}
\label{note:REF}
{\rm Since the rule $\rightarrow \psi $ is identified with ${\sf T}\rightarrow \psi$,  [{\sf MP}] implicitly implies Reflexivity {\sf [Ref]}: 
$\displaystyle \frac{\rightarrow \psi }{\psi}$.
}\end{note}

\begin{definition}
\label{def:LP2ABF}
{\rm  The assumption-based argumentation framework that is {\em induced} by a disjunctive logic program $\pi$ is defined by:
$\ABF(\pi) = \tup{{\mathfrak L},\pi,\sim\!\Atoms(\pi),-}$, where $-\!\sim\!p = \{p\}$ for every $p \in \Atoms(\pi)$.
} \end{definition}

\begin{example}
\label{ex:translations-0}
{\rm Let $\pi_2= \{{\rightarrow\! p \lor q}, \ {p \rightarrow q}, \ {q \rightarrow p}\}$. The attack diagram of  the induced assumption-based argumentation
framework $\ABF(\pi_2)$ is shown in Figure~\ref{fig:ex:1}. 

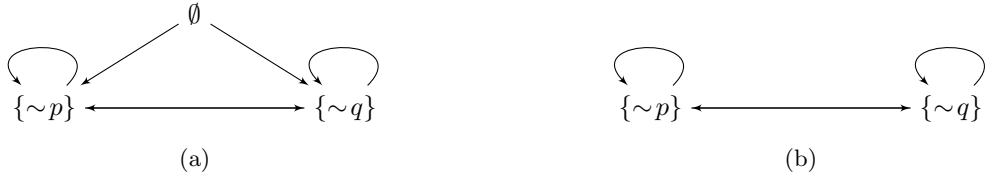
\begin{figure}[hbt]
\centering
\begin{subfigure}{0.5\textwidth}
\centering
\begin{tikzpicture}
\tikzset{edge/.style = {->,> = latex'}}

\node (sima) at (0,0) {$\{\sim\! p\}$};
\node (simb) at (4,0) {$\{\sim\! q\}$};
\node (empty) at (2,1.25) {$\emptyset$};

\draw[edge, loop, looseness=4] (sima) to (sima);
\draw[edge, loop, looseness=4] (simb) to (simb);
\draw[edge] (sima) to (simb);
\draw[edge] (simb) to (sima);

\draw[edge] (empty) to (simb);
\draw[edge] (empty) to (sima);
\end{tikzpicture}
\caption{}
\label{fig:ex:1}
\end{subfigure}~
\begin{subfigure}{0.5\textwidth}
\centering
\begin{tikzpicture}
\tikzset{edge/.style = {->,> = latex'}}

\node (sima) at (0,0) {$\{\sim\! p\}$};
\node (simb) at (4,0) {$\{\sim\! q\}$};
\node (empty) at (2,1.25) {$\emptyset$};

\draw[edge, loop, looseness=4] (sima) to (sima);
\draw[edge, loop, looseness=4] (simb) to (simb);
\draw[edge] (sima) to (simb);
\draw[edge] (simb) to (sima);
\end{tikzpicture}
\caption{}
\label{fig:ex:2}
\end{subfigure}
\caption{Attack diagrams for Examples~\ref{ex:translations-0}, \ref{ex:translations}c  (left) and  Example~\ref{ex:translations}b (right)}
\end{figure}

In the notations of Definition~\ref{def:ABF}, we have: $\Gamma = \pi_2$ and $\Lambda = \{{\sim\!p},{\sim\!q}\}$. To see, e.g., that the set $\{\sim\!p\}$ attacks itself, 
note that by {\sf [MP]} on $\rightarrow p \vee q $ we conclude $\Gamma \vdash p \vee q$, and by {\sf [Res]} it holds that $\Gamma,{\sim\!p} \vdash q$. 
Thus, since  $q\rightarrow p \in \Gamma$, by {\sf [MP]} we get $\Gamma,{\sim\!p} \vdash p$, namely: $\Gamma,{\sim\!p} \vdash {-\!\sim\!p}$.
From similar reasons $\{{\sim\!q}\}$ attacks itself. For the attacks of $\emptyset$ on $\{{\sim\!p}\}$ and $\{{\sim\!q}\}$ we also need {\sf [RBC]}.

Note that in this case $\{p,q\}$ is the stable model of $\pi_2$ (Definition~\ref{def:stable-model}) and $\emptyset$ is the stable extension of 
$\ABF(\pi_2)$ (Definition~\ref{def:semantics}).
} \end{example}

\begin{example}
\label{examp:3}
{\rm Logic programs and their induced ABFs from earlier examples are the following:
\begin{itemize}
     \item [a)] The assumption-based argumentation framework considered in Example~\ref{examp:1} is induced from the logic program 
                     $\{{\sim\!q} \rightarrow p\}$.\footnote{In this case $\mathfrak{L}_{\sf MP}$ is used instead of the extended logic $\mathfrak{L}$, but as shown
                     in~\cite{caminada2017equivalence,CS18}, for normal logic programs this representation is sufficient for the correspondence between LPs and the induced 
                     ABA frameworks.} 
    \item [b)] Figure~\ref{fig:ex:0} depicts (a fragment of) the graph of the assumption-based argumentation framework that is induced by the logic program 
                  $\pi_1 = \{{\sim\!p} \rightarrow q \lor r \}$  from Example~\ref{examp:2}. This framework has two stable extensions: ${\cal E}_1 = \{\sim\!p,\sim\!q,\}$ 
                  and ${\cal E}_2 = \{\sim\!p,\sim\!r,\}$.

                  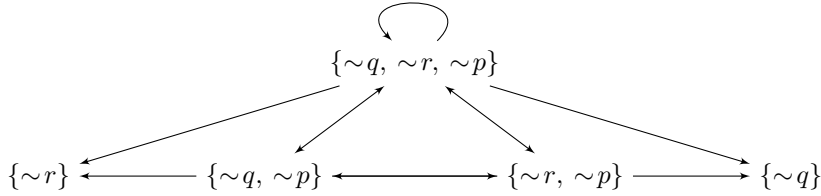
\begin{figure}[hbt]
                  \centering
                  \begin{tikzpicture}
                  \tikzset{edge/.style = {->,> = latex'}}

                  \node (simqrp) at (2,1.5) {$\{{\sim\!p},\: \sim\!q,\: {\sim\!r}\}$};
                  \node (simqp) at (0,0) {$\{{\sim\!p},\: {\sim\!q}\}$};
                  \node (simrp) at (4,0) {$\{{\sim\!p},\: {\sim\!r}\}$};
                  \node (simr) at (-3,0) {$\{{\sim\!r}\}$};
                  \node (simq) at (7,0) {$\{{\sim\!q}\}$};

                  \draw[edge, loop, looseness=4] (simqrp) to (simqrp);
                  \draw[edge] (simqrp) to (simq);
                  \draw[edge] (simqrp) to (simr);
                  \draw[edge,style={<->}] (simqrp) to (simqp);
                  \draw[edge,style={<->}] (simqrp) to (simrp);
                  \draw[edge] (simqp) to (simrp);
                  \draw[edge] (simqp) to (simr);
                  \draw[edge] (simrp) to (simqp);
                  \draw[edge] (simrp) to (simq);

                  \end{tikzpicture}
                  \caption{Attack diagrams for Examples~\ref{examp:2} and~\ref{examp:3}}
                  \label{fig:ex:0}
                  \end{figure}
\end{itemize}
} \end{example}

\begin{note}
\label{note:flat-translation}
{\rm The translation in Definition~\ref{def:LP2ABF} always gives rise to a so-called \emph{flat\/} ABF, that is, an ABF for which there is no 
$\Delta \subseteq \Lambda$ and $\psi \in \Lambda\setminus \Delta$ such that $\Gamma,\Delta\vdash \psi$. This is shown in the following proposition:

\begin{proposition}
For every disjunctive logic program $\pi$ and the induced assumption-based argumentation framework $\ABF(\pi) = \tup{\mathfrak{L}, \Gamma, \Lambda, -} =
\tup{{\mathfrak L},\pi,\sim\!\Atoms(\pi),-}$, if $\Delta\subseteq \Lambda$ and 
$\Gamma,\Delta\vdash \psi$, then $\psi \not\in \Lambda\setminus \Delta$.
\end{proposition}

\begin{proof}
Since $\Lambda$ consists only of formulas of the form ${\sim\!p}$, we can restrict our attention to such formulas. Suppose that $\Gamma,\Delta\vdash \psi$ 
for some $\Delta\subseteq \Lambda$. Since $\Lambda$ contains only formulas of the form ${\sim\!p}$, and since $\Gamma = \pi$, then by Note~\ref{note:for:flatness}, 
if $\Gamma,\Delta \vdash \:\sim\!p$, necessarily $\sim\!p\in\Delta$. \qed
\end{proof}
} \end{note}

To relate the semantics of logic programs and their induced ABFs, we need the following notations:

\begin{definition}
\label{def:main-notations}
{\rm Let $\pi$ be a disjunctive logic program and let $\Theta \subseteq \Atoms(\pi)$. We denote: 
\begin{itemize}
     \item $\floor{\sim\!\Theta} = \Theta$ (Thus $\floor{\cdot}$ eliminates the leading $\sim$ from the formulas in the set) \smallskip
     \item If $\Delta \subseteq \:\sim\!\Atoms(\pi)$ then $\underline{\Delta} = \Atoms(\pi) \setminus \floor{\Delta}$ \smallskip
     \item If $\Delta \subseteq \Atoms(\pi)$ then $\overline{\Delta} = \:\sim\!(\Atoms(\pi) \setminus \Delta$)  
\end{itemize}
In other words, $\underline{\Delta}$ (respectively, $\overline{\Delta}$) takes the complementary set of $\Delta$ and
removes (respectively, adds) the negation-as-failure operator from (respectively, to) the prefix of its formulas. 
} \end{definition}

\begin{example}
{\rm Consider again the program $\pi_2$ of Example~\ref{ex:translations-0}, and let $\Delta = {\sim\!\Atoms}(\pi_2) = \{{\sim\!p},{\sim\!q\}}$.
Then $\underline{\Delta} = \emptyset$ and  $\overline{\emptyset} = \Delta$. 
} \end{example}

The semantic correspondence between a logic program and the induced ABF is obtained by the following results:
\begin{enumerate}
    \item If $\Delta$ is a stable extension of $\ABF(\pi)$, then $\underline{\Delta}$ is a stable model of $\pi$, and
    \item If $\Delta$ is a stable model of $\pi$, then $\overline{\Delta}$ is a stable extension of $\ABF(\pi)$.
\end{enumerate}

Before showing these results, we first provide some examples and notes concerning related results.

\begin{example}
\label{ex:translations-0b}
{\rm Let's revisit our previous examples and see that the correspondence between the stable models of the logic programs and the stable extensions of the induced
ABFs, indicated above, indeed holds in these examples:
\begin{itemize}
    \item [a)] The stable models $M_3$ and $M_4$ of the logic program $\pi_1$ in Example~\ref{examp:2} correspond to the stable extensions ${\cal E}_1$ and ${\cal E}_2$
                   of the assumption-based framework $\ABF(\pi_1)$ that is induced from $\pi_1$  according to Definition~\ref{def:LP2ABF}, which is considered in Example~\ref{examp:3}(b) 
                   and Figure~\ref{fig:ex:0}. Indeed, $M_3 = \underline{{\cal E}_2}$ and $M_4 = \underline{{\cal E}_1}$ (Alternatively, ${\cal E}_1 = \overline{M_4}$ and 
                   ${\cal E}_2 = \overline{M_3}$).
    \item [b)] As shown in Example~\ref{ex:translations-0}, the correspondence between the stable model of $\pi_2$ and the stable extension of the framework $\ABF(\pi_2)$ that
                   is induced from $\pi_2$ (see Figure~\ref{fig:ex:1}), is preserved.  
\end{itemize}
As indicated before this example, and as  we show in Section~\ref{sec:correctness-proofs} (Propositions~\ref{prop:model-to-extension} and~\ref{prop:extension-to-model}), 
the two items above are not a coincidence.
} \end{example}

\begin{example}
\label{ex:translations}
{\rm As noted in the introduction, Caminada and Schulz~\cite{caminada2017equivalence,CS18} consider the correspondence between ABA frameworks and normal logic programs. 
In our notations, the ABF that they associate with a normal logic program $\pi$ is $\ABF_{{\sf Norm}}(\pi) = \tup{{\mathfrak L}_{\sf MP},\pi,{\sim\!\Atoms(\pi)},-}$ 
constructed as in Definition~\ref{def:LP2ABF}, except that ${\mathfrak L}_{\sf MP}$ is defined by Modus~Ponens only. 
\begin{itemize}
     \item [a)] To see that $\ABF_{{\sf Norm}}$ is not adequate for disjunctive logic programs, consider the program $\pi_3=\{{\rightarrow p \lor q}\}$. 
             This program has two stable models: $\{p\}$ and $\{q\}$. However, the only stable extension of $\ABF_{{\sf Norm}}(\pi_3)$ is 
              $\{\sim\!p, \sim\!q\}$. We can enforce $\{\sim\!p\}$ and $\{\sim\!q\}$ being stable models by requiring that $\pi_3\cup\{\sim\!p\}\vdash q$ 
              and $\pi_3\cup\{\sim\!q\}\vdash p$. For this, we need the resolution rule ${\sf [Res]}$. \smallskip
     \item [b)] Adding only ${\sf [Res]}$ to  ${\mathfrak L}_{\sf MP}$ (i.e, without ${[\sf RBC]}$) as the inference rules for the logic 
              is yet  not sufficient. To see this, consider again the program $\pi_2$ from Example~\ref{ex:translations-0}.
              Recall that $\{p,q\}$ is the sole stable model of $\pi_2$. The attack diagram of the ABF based on {\sf [MP]} and {\sf [Res]} is shown in 
              Figure~\ref{fig:ex:2}, thus there is no stable extension in this case (However, if ${[\sf RBC]}$ is also available, we get the ABF depicted in Figure~\ref{fig:ex:1}, 
              which, as indicated in Example~\ref{ex:translations-0b}b, {\em is\/} a faithful translation of $\pi_2$).
\end{itemize}
} \end{example}

\subsection{Proof of Correctness}
\label{sec:correctness-proofs}

The correctness of the translation follows from Propositions~\ref{prop:model-to-extension} and~\ref{prop:extension-to-model} below. First, we need some 
definitions and lemmas. 
In what follows ${\mathfrak L}$ denotes the logic defined in Section~\ref{sec:translation}, and $\pi$ is an arbitrary disjunctive logic program.

We start with the following soundness and completeness result.

\begin{lemma}
\label{prop:sem:theories}
Let $\Delta$ be a set of ${\mathcal L}$-formulas of the form ${\sim\!p}$ or $q_1,\ldots,q_m,{\sim\!r_1},\ldots,{\sim\!r_k}  \rightarrow p_1\lor \ldots\lor p_n$. Then:
\begin{itemize}
     \item [a)] If $\psi \in {\sf Cn}_{\mathfrak L}(\Delta)$ then $M\models \psi$ for every $M$ such that $M \models \Delta$. 
     \item [b)] If $\psi=r_1\lor \ldots \lor r_m$ and $M\models \psi$ for every $M$ such that $M \models \Delta$, then $\psi \in {\sf Cn}_{\mathfrak L}(\Delta)$.
\end{itemize}
\end{lemma}

\begin{note}
{\rm Part~(b) of Lemma~\ref{prop:sem:theories} does not hold for {\em any formula\/} $\psi$, but only for a disjunction of atoms. 
To see this, let $\Delta=\{{\sim\!s}, \:{p\rightarrow s}\}$. The only $M \subseteq \{p,s\}$ such that $M \models \Delta$ is $M=\emptyset$. 
Thus, for every $M$ such that $M \models \Delta$ it holds that $M \models \:\sim\!p$. However, $\sim\!p$ cannot be derived from $\Delta$
using  ${\sf [MP]}$, ${\sf [Res]}$ and ${\sf [RBC]}$.
} \end{note}

\begin{proof}
We prove Part~(a) of the lemma by induction on the number of applications of the inference rules in the derivation of 
$\psi \in {\sf Cn}_{\mathfrak L}(\Delta)$. 

For the base step, no inference rule is applied in the derivation 
of $\psi$, thus $\psi \in \Delta$. Since $M \models\Delta$, we have that $M \models \psi$.

For the induction step, we consider three cases, each one corresponds to an application of a different inference rule in the last step of the 
derivation of $\psi$:

\begin{enumerate}
\item Suppose that the last step in the derivation
of $\psi$ is an application of Resolution. Then $\psi = p'_1 \vee \ldots \vee p'_m \vee \ldots \vee p''_1 \vee \ldots \vee p''_k$ is
obtained by {{\sf [Res]}} from $p'_1 \vee \ldots \vee p'_m \vee q_1 \vee \ldots \vee q_n \vee  p''_1 \vee \ldots \vee p''_k$ and 
${\sim\!q_i}$ $(i=1,\ldots,n)$. Suppose that $M \models \Delta$. Since ${\sim\!q_i} \in \CNL(\Delta)$ iff ${\sim\!q_i} \in \Delta$ and 
since $M\models \Delta$, we have $M \models {\sim\!q_i}$ $(i=1,\ldots,n)$, thus $M \not\models q_i$ $(i=1,\ldots,n)$. 
By the induction hypothesis, $M \models p'_1 \vee \ldots \vee p'_m \vee q_1 \vee \ldots \vee q_n \vee  p''_1 \vee \ldots \vee p''_k$. 
By Definition~\ref{def:models}, then, $M \models p'_i$ for some $1 \leq i \leq m$, or $M \models p''_j$ for some $1 \leq j \leq k$. 
By Definition~\ref{def:models} again, $M \models \psi$.

\item Suppose that the last step in the derivation of $\psi$ is an application of Reasoning by Cases, and let $M \models \Delta$. By induction hypothesis,
$M \models p_1 \vee \ldots \vee p_n$, and $M \models \psi$ in case that $M \models p_j$ for some $1 \leq j \leq n$.
But by Definition~\ref{def:models} the former assumption means that there is some $1 \leq j \leq n$ for which $M \models p_j$, therefore $M \models\psi$.

\item Suppose that the last step in the derivation of $\psi$ is an application of Modus~Ponens, and let $M \models \Delta$. 
By induction hypothesis $M \models l_i$, where $l_i \in \{p_i,{\sim\!p_i}\}$ for $i =1,\ldots,n$. Thus, by Definition~\ref{def:models},
$M \models l_1, \ldots, l_n$. On the other hand, by induction hypothesis again, $M \models l_1, \ldots, l_n \rightarrow  \psi $. 
By Definition~\ref{def:models}, 
$M \models \psi$.
\end{enumerate}

\noindent We now turn to Part~(b) of the lemma. 
Suppose that $M\models \psi$ for every $M$ such that
$M \models \Delta$, yet $\psi \not\in {\sf Cn}_{\mathfrak L}(\Delta)$ (where $\psi=r_1\lor \ldots \lor r_m$ for some $m\geqslant 1$). 
We show that this leads to a contradiction by constructing an $M'$ for which $M' \models \Delta$ but $M' \not\models\psi$.
For this, we consider the following set of the minimal disjunctions of a set of formulas ${\cal S}$:
\[ {\sf MD}({\cal S})=\{ q_1\lor \ldots\lor q_n \in {\cal S} \mid \not\exists \{i_1,\ldots,i_m\}\subsetneq \{1,\ldots,n\} \text{ s.t. }
    q_{i_1}\lor\ldots\lor q_{i_m} \in {\cal S} \}. \]
We first show that if $q_1\lor\ldots\lor q_n \in {\sf MD}({\sf Cn}_{\mathfrak L}(\Delta))$ then there is an $1\leqslant i\leqslant n$ such that 
$q_i \not\in \{r_1,\ldots,r_m\}$ and ${\sim\!q_i} \not\in \Delta$. Indeed, suppose first for a contradiction that $q_1\lor\ldots\lor q_n \in 
{\sf MD}({\sf Cn}_{\mathfrak L}(\Delta))$, yet  for every $1<i\leqslant n$ either $q_i\in  \{r_1,\ldots,r_m\}$ or ${\sim\!q_i} \in \Delta$.
In that case, by ${{\sf [Res]}}$, $r_1\lor \ldots\lor r_m \in {\sf Cn}_{\mathfrak L}(\Delta)$, contradicting the original supposition that 
$r_1\lor \ldots\lor r_m  \not\in {\sf Cn}_{\mathfrak L}(\Delta)$. Suppose now, again towards a contradiction, that $q_1\lor\ldots\lor q_n \in 
{\sf MD}({\sf Cn}_{\mathfrak L}(\Delta))$, yet for every $1\leqslant i\leqslant n$, ${\sim\!q_i} \in \Delta$. In that case, by ${{\sf [Res]}}$ again, 
$q_i\in {\sf Cn}_{\mathfrak L}(\Delta)$ (for every $i$), but this, together with the assumption the $\sim\!q_i \in \Delta$ (for every $i$), 
contradicts the assumption that there is an $M$ such that $M \models \Delta$.

We thus showed that in any case, if $q_1\lor\ldots\lor q_n \in {\sf MD}({\sf Cn}_{\mathfrak L}(\Delta))$, then there is an 
$1\leqslant i\leqslant n$ such that $q_i \not\in \{r_1,\ldots,r_m\}$ and $\sim\!q_i \not\in \Delta$.

We now construct the model $M'$ such that $M'\models\Delta$ and $M'\not\models r_1\lor \ldots \lor r_m$: Let $M'$ contain exactly one 
$q_i$ with $1\leqslant i\leqslant n$ and $q_i\not\in \{r_1,\ldots,r_m\}$ and ${\sim\!q_i} \not\in\Delta$ for every 
$q_1\lor\ldots\lor q_n \in {\sf MD}({\sf Cn}_{\frak L}(\Delta))$. (If there is more than one such $i$, take $i$ which is minimal among 
$1\leqslant i\leqslant n$.) As shown above, there is at least one such $i$ for every formula 
$q_1\lor\ldots\lor q_n\in {\sf MD}({\sf Cn}_{\mathfrak{L}}(\Delta))$. 

We now show that (1)~$M'\models\Delta$ and (2)~$M'\not\models r_1\lor\ldots\lor r_m$.

Item~(1): Suppose that $\phi\in {\sf Cn}_{\mathfrak{L}}(\Delta)$. We have to consider two possibilities: $\phi= {\sim\! s}$ or 
$\phi=q_1\lor\ldots\lor q_n$. In the first case, by construction, $s\not\in M'$ and thus $M'\models \:\sim\!s$.
In the second case, there is a $q_{i_1}\lor\ldots\lor q_{i_m}\in {\sf MD}({\sf Cn}_{\mathfrak L}(\Delta))$ such that
$\{i_1,\ldots,i_m\}\subseteq \{1,\ldots,n\}$. By construction, there is a $1\leqslant j\leqslant m$ such that $q_{i_j}\in M'$. 
Thus, $M' \models q_1\lor\ldots\lor q_n$.

Item~(2): By construction $r_i\not\in M'$ $(i = 1,\ldots,m)$, thus $M'\not\models r_1\lor \ldots \lor r_m$. \qed
 \end{proof}

\begin{lemma}
\label{lemma:no:bigger:set:cn}
Let $\pi$ be a logic program. For every sets $M,N \subseteq \Atoms(\pi)$, if $N\setminus M \neq \emptyset$ then $N\not\models \CNL(\overline{M}\cup\pi)$.
\end{lemma}

\begin{proof}
Suppose that $N\setminus M\neq \emptyset$ and let $p \in N\setminus M$, Then ${\sim\!p} \in \overline{M}$ and thus 
${\sim\!p} \in \CNL(\overline{M}\cup\pi)$. Since $N \models p$, $N \not\models {\sim\!p} $, and so $N\not\models \CNL(\overline{M}\cup\pi)$. \qed
\end{proof}

\begin{lemma}
\label{lem:the:hard:one} 
Given a logic program $\pi$, if $M$ is a minimal model of a logic program $\pi'\subseteq \pi^M$, then for every $N \subsetneq M$, $N\not\models \CNL(\overline{M}\cup\pi)$.
\end{lemma}

\begin{proof}
Suppose that $M$ is a minimal model of a logic program $\pi'\subseteq \pi^M$, and suppose towards a contradiction that there is some $N \subsetneq M$ such that
$N \models \CNL(\overline{M}\cup\pi)$. We show that $N$ is a model of $\pi'$ by showing that $N$ is a model of $\pi^{M}$. Indeed, let
$p_1,\ldots,p_n,{\sim\!q_1},\ldots,{\sim\!q_m} \rightarrow r_1 \vee \ldots \vee r_k\in \pi$ such that $q_1,\ldots,q_m \not\in M$. 
Then ${\sim\!q_i} \in \overline{M}$ for every $i=1,\ldots,m$. Suppose furthermore that $ p_1,\ldots,p_n \in N$. 
Then (since $N\subset M$), also $p_1,\ldots,p_n \in M$ and thus $M \vdash r_1\lor\ldots\lor r_k$, which implies that $r_1\lor\ldots\lor r_k \in \CNL(\overline{M}\cup\pi)$. 
It follows that $r_i \in N$ for some $i=1\ldots,k$ (as we assumed that $N \models \CNL(\overline{M}\cup\pi)$), and so $N$ is a model of the rule  
$p_1,\ldots,p_n \rightarrow r_1 \vee \ldots \vee r_k \in \pi^M$. This means that $N$ is a model of $\pi'$, which contradicts the minimality of $M$. \qed
\end{proof}

\begin{lemma}
\label{lemma:min:model:of:cn}
Let $M$  be a stable model  of $\pi$.
Then $M = \min_\subseteq \{ N \subseteq \Atoms(\pi) \mid N \models {\sf Cn}_{\mathfrak L}(\pi \cup \overline{M}) \}$.
\end{lemma}

\begin{proof}
Let $M$ be a stable model of $\pi$. We first show that $M\models \CNL(\overline{M}\cup\pi)$. Let $\psi \in \CNL(\overline{M}\cup\pi)$. 
Then it has an ${\mathfrak L}$-derivation ${\sf D}_{\mathfrak L}(\psi)$.\footnote{Namely, ${\sf D}_{\mathfrak L}(\psi)$ is a finite sequence
$\tup{T_1,\ldots,T_n}$ of proof tuples, where $T_n = \psi$, and for each $1 \leq i \leq n$ $T_i$ is either an element of $\overline{M}\cup\pi$, or
is obtained by an application of one of the inference rules of ${\mathfrak L}$ on proof tuples in $\{T_1,\ldots,T_{i-1}\}$. The size of ${\sf D}_{\mathfrak L}(\psi)$
is $n$.} We show by induction on the size of ${\sf D}_{\mathfrak L}(\psi)$ that $M \models \psi$.

For the base step, no inference rule is applied in ${\sf D}_{\mathfrak L}(\psi)$, thus $\psi = {\sim\!\phi} \in \overline{M}$. Since this means that $\phi\not\in M$, 
we have that $M \models \psi$.

For the induction step, we consider three cases, each one corresponds to an application of a different inference rule in the 
last step of ${\sf D}_{\mathfrak L}(\psi)$:
\begin{enumerate}
\item Suppose that the last step in ${\sf D}_{\mathfrak L}(\psi)$ is an application of Resolution. 
Then $\psi = p'_1 \vee \ldots \vee p'_m \vee \ldots \vee p''_1 \vee \ldots \vee p''_k$ is
obtained by {[{\sf Res}]} from $p'_1 \vee \ldots \vee p'_m \vee q_1 \vee \ldots \vee q_n \vee  p''_1 \vee \ldots \vee p''_k$ and 
${\sim\!q_i}$ $(i=1,\ldots,n)$. Since ${\sim\!q_i} \in {\sf Cn}_{\mathfrak L}(\overline{M}\cup\pi)$ means that ${\sim\!q_i} \in \overline{M}$, 
we have $q_i\not\in M$ for every $i= 1,\ldots,n$. By the induction hypothesis, 
$M \models  p'_1 \vee \ldots \vee p'_m \vee q_1 \vee \ldots \vee q_n \vee  p''_1 \vee \ldots \vee p''_k$. 
Thus, by Definition~\ref{def:models}, $M \models p'_i$ for some $1 \leq i \leq m$, or 
$M \models p''_j$ for some $1 \leq j \leq k$. By Definition~\ref{def:models} again, $M \models \psi$.

\item Suppose that the last step in  ${\sf D}_{\mathfrak L}(\psi)$ is an application of Reasoning by Cases. 
By induction hypothesis we know that $M \models p_1 \vee \ldots \vee p_n$, and that $M \models \psi$ in case that $M \models p_j$ 
for some $1 \leq j \leq n$. But by Definition~\ref{def:models} the former assumption means that there is some $1 \leq j \leq n$ for 
which $M \models p_j$, therefore $M \models\psi$.

\item Suppose that the last step in ${\sf D}_{\mathfrak L}(\psi)$ is an application of Modus Ponens 
on $p_1,\ldots,p_n,{\sim\!q_1},\ldots,{\sim\!q_m} \allowbreak \rightarrow r_1\lor\ldots\lor r_k \in \pi$. By induction hypothesis $M \models p_i$, 
for every $1\leqslant i\leqslant n$. Also, for every $1\leqslant i\leqslant m$, ${\sim\!q_i} \in \CNL(\overline{M}\cup\pi)$ implies $q_i\not \in M$. 
Thus, $p_1,\ldots,p_n\rightarrow  r_1\lor\ldots\lor r_{k} \in \pi^M$. Since $M$ is a model of $\pi^M$, $M\models r_i$ for some $1\leqslant i\leqslant k$.
Thus, $M\models r_1\lor\ldots\lor r_k$.
\end{enumerate}
We have shown that $M\models \CNL(\overline{M}\cup\pi)$. By Lemma~\ref{lemma:no:bigger:set:cn}, for no $N\subseteq \Atoms(\pi)$ 
such that $N\setminus M\neq \emptyset$ it holds that $N\models \CNL(\overline{M}\cup\pi)$.
Thus, if there is some $N\subseteq \Atoms(\pi)$ such that $N\models Cn_{\frak L}(\overline{M}\cup\pi)$, then $N\subseteq M$. 
But if $N\subset M$, by Lemma~\ref{lem:the:hard:one} we have that $N\not \models  \CNL(\overline{M}\cup\pi)$. Thus, $M$ is the 
unique subset of $\Atoms(\pi)$ which is a model of $ \CNL(\overline{M}\cup\pi)$. \qed
\end{proof}

\begin{corollary}
\label{corol:pink}
Let $M$ be a stable model of $\pi$. Then $p \in M$ iff $p \in\CNL(\overline{M}\cup\pi)$. 
\end{corollary}

\begin{proof}
Suppose first that $p\in M$. Thus $M\models p$. Since  by Lemma \ref{lemma:min:model:of:cn} $M$ is the unique model of $\CNL(\overline{M}\cup \pi)$, 
by Lemma~\ref{prop:sem:theories} it holds that $M\models p$ implies that $p \in \CNL(\overline{M}\cup\pi)$. For the converse, suppose
that $p \in \CNL(\overline{M}\cup\pi)$. By Lemma~\ref{lemma:min:model:of:cn}, $M\models \CNL(\overline{M}\cup\pi)$, and so $M\models p$. \qed
\end{proof}

\begin{example}
{\rm Corollary~\ref{corol:pink} actually says that any atom $p$ that is verified by a 2-valued sable model $M$ of $\pi$, can be derived by 
assuming that all the assumptions not in $M$ are false (and using the derivation rules together with the rules of the program). To illustrate this, take 
for example the program $\{{\sim\!q}\rightarrow p_1\lor p_2\}$. This program has two stable models: $\{p_1\}$ and $\{p_2\}$. The result above then says 
that $p_1$ belongs to $\CNL(\overline{\{p_1\}}\cup\pi)=\CNL(\{{\sim\!p_2},{\sim\!q}\}\cup\pi)$. Indeed, it holds that $p_1$ is derivable from 
$\{{\sim\!p_2},{\sim\!q}\}\cup\pi$ using the rules $\{{\sim\!q}\}\cup \{{\sim\!q}\rightarrow p_1\lor p_2\} \vdash p_1\lor p_2$ and $\{p_1\lor p_2,{\sim\!p_2}\}\vdash p_1$.
} \end{example}

\begin{lemma}
Let $\pi$ be a disjunctive logic program, $\Delta \subseteq {\sim\!\Atoms}(\pi)$, and $r_1\lor\ldots\lor r_k \in {\sf Cn}_{\mathfrak L}(\pi \cup \Delta)$. 
If $M$ is a model of $\pi^{\floor{\Delta}}$ and $M \subseteq \floor{\Delta}$, then $r_i\in M$ for some $1\leqslant i\leqslant k$.
\end{lemma}

\begin{proof}
We show by induction on the number of steps used in deriving $r_1\lor\ldots \lor r_k$ from $\pi \cup\Delta$, that  
$r_1\lor\ldots \lor r_k \in {\sf Cn}_{\mathfrak L}(\pi \cup \Delta)$  implies that $ M \models r_1\lor\ldots \lor r_k$.
This means that $M \models r_i$ for some $1\leqslant i\leqslant k$, and so $r_i\in M$ for that $i$.

For the base step, since every element of $\Delta$ is of the form ${\sim\!p}$, we have that $r_1\lor\ldots \lor r_k \not\in \Delta$.
Thus, $r_1\lor\ldots\lor r_k$ is obtained from some rule ${\sim\!q_1},\ldots,{\sim\!q_m} \rightarrow r_1\lor\ldots \lor r_k \in \pi$,
where ${\sim\!q_1},\ldots,{\sim\!q_m} \in\Delta$. In that case, since (1)~$M \subseteq \underline{\Delta}$; (2)~$M$ is a model of 
$\pi^{\floor{\Delta}}$; and (3)~$\rightarrow r_1\lor\ldots \lor r_k \in \pi^{\floor{\Delta}}$, by 
Definition~\ref{def:stable-model} there is a $1\leqslant i\leqslant k$ such that $r_i\in M$.

For the induction step, suppose that the claim holds for every $r_1\lor\ldots \lor r_k$  that is derived from $\Delta$ using $n$ or 
less derivation steps. We consider three cases:
\begin{itemize}
\item $r_1\lor\ldots\lor r_k$ is obtained by applying {\sf[MP]} to the rule $p_1,\ldots,p_n,{\sim\!q_1},\ldots,{\sim\!q_m} \rightarrow r_1\lor\ldots\lor r_k$. 
In this case, ${\sim\!q_1},\ldots,{\sim\!q_m}\in \Delta$ and so $q_i \not\in \underline{\Delta}$ for any $1\leqslant i\leqslant m$.  It follows that 
$p_1,\ldots,p_n\rightarrow r_1\lor\ldots\lor r_k \in \pi^{\floor{\Delta}}$. Let now $M$ be a model of $\pi^{\floor{\Delta}}$ such that $M \subseteq \floor{\Delta}$.  
By the induction hypothesis, $p_i \in {\sf Cn}_{\mathfrak L}(\pi \cup \Delta)$ implies that $p_i \in M$ (for every $1\leqslant i\leqslant n$). Thus, since $M$ is a model
of $\pi^{\floor{\Delta}}$, $r_i \in M$ for some $1 \leqslant i \leqslant k$.

\item $r_1\lor\ldots\lor r_k$ is obtained by applying {\sf[Res]} from $r_1\lor\ldots\lor r_k\lor r_{k+1}\lor \ldots r_{n}$ and 
${\sim\!r}_{k+1},\ldots,{\sim\!r}_n$. Suppose furthermore that $M \subseteq \floor{\Delta}$ is a model of $\pi^{\floor{\Delta}}$. 
By the induction hypothesis and since $r_1\lor\ldots\lor r_k\lor r_{k+1}\lor \ldots r_{n} \in {\sf Cn}_{\mathfrak L}(\pi \cup \Delta)$, we have that
$r_i\in M$ for some $1\leqslant i\leqslant n$. Since $\sim\!r_i\in\Delta$ for $k+1\leqslant i\leqslant n$ and $M\subseteq \floor{\Delta}$, 
$r_i\not\in M$ for every $k+1\leqslant i \leqslant n$. This means that $r_i\in M$ for some $1\leqslant i\leqslant k$.

\item $r_1\lor\ldots\lor r_k$ is obtained by applying {\sf [RBC]}, since $s_1\lor\ldots\lor s_n \in {\sf Cn}_{\mathfrak L}(\pi \cup \Delta)$ and since
$r_1\lor\ldots\lor r_k  \in {\sf Cn}_{\mathfrak L}(\pi \cup {\{\rightarrow s_i\}} \cup \Delta)$ for every $1\leqslant i\leqslant n$. By the induction hypothesis, 
($\dagger$):~for every  $1\leqslant i\leqslant n$ and every model $M'$ of $(\pi \cup {\{\rightarrow s_i\}})^{\floor{\Delta}}$ such that 
$M' \subseteq \underline{\Delta}$, we have that $r_1\lor\ldots\lor r_k \in {\sf Cn}_{\mathfrak L}(\pi \cup {\{\rightarrow s_i\}} \cup \Delta)$ 
implies $r_i\in M'$ for some $1\leqslant i\leqslant k$. Now, since $s_1\lor\ldots\lor s_n \in  {\sf Cn}_{\mathfrak L}(\pi \cup \Delta)$, by the induction
hypothesis this means that $s_j\in M$ for some $1\leqslant j\leqslant n$ for every model $M$ of $\pi^{\floor{\Delta}}$. In other words, $M$ is a
model of $\pi^{\floor{\Delta}}$ iff it is a model of $(\pi \cup \{\rightarrow s_j\})^{\floor{\Delta}}=\pi^{\floor{\Delta}}\cup\{\rightarrow s_j\}$ 
for some $1\leqslant j\leqslant n$. By $(\dagger)$, then, $r_i\in M$ for some $1\leqslant i\leqslant k$. \qed
\end{itemize}
\end{proof}

By the last lemma, in case that $k=1$, we therefore have:

\begin{corollary}
\label{corol:black}
Let $\pi$ be a disjunctive logic program, $\Delta \subseteq {\sim\!\Atoms}(\pi)$ and $r \in {\sf Cn}_{\mathfrak L}(\pi \cup \Delta)$.
If $M$ is a model of $\pi^{\floor{\Delta}}$ and $M \subseteq \floor{\Delta}$, then $r \in M$.
\end{corollary}

Now we can show the main results of this section.

\begin{proposition}
\label{prop:model-to-extension}
If $M$ is a stable model of $\pi$, then $\overline{M}$ is a stable extension of $\ABF(\pi)$.
\end{proposition}

\begin{proof}
Suppose that $M$ is a stable model of $\pi$. We show first that $\overline{M}$ is conflict-free in $\ABF(\pi)$. Otherwise, 
there is some $\sim\!p \in \overline{M}$ such that $\pi,\overline{M} \vdash p$. The former implies that $p \not\in M$. 
But since $M$ is a model of $\pi^M$, by Corollary~\ref{corol:black}, $p \in {\sf Cn}_{\mathfrak L}(\pi \cup \overline{M})$ implies that $p \in M$, 
a contradiction to $p \not\in M$

We now show that $\overline{M}$ attacks every ${\sim\!p} \in \:\sim\!{\cal A}(\pi) \setminus \overline{M}$. 
This means that we have to show that $\pi,\overline{M} \vdash p$ for every $p \in M$. This follows from Corollary~\ref{corol:pink}. \qed
\end{proof}

\begin{proposition}
\label{prop:extension-to-model}
If ${\cal E}$ is a stable extension of $\ABF(\pi)$, then $\underline{{\cal E}}$ is a stable model of $\pi$.
\end{proposition}

\begin{proof}
We first show that $\underline{{\cal E}}$ is a model of $\pi^{\underline{{\cal E}}}$. Indeed, let 
$p_1,\ldots,p_n,\sim\!q_1,\ldots, \sim\!q_m\rightarrow r_1\lor \ldots\lor r_k\in\pi$.
If $q_j \in \underline{{\cal E}}$ for some $1 \leq j \leq m$ we are done: the rule is satisfied by $\underline{{\cal E}}$. Otherwise,
$q_1,\ldots,q_m \not\in \underline{{\cal E}}$, and so $p_1,\ldots,p_n\rightarrow r_1\lor\ldots\lor r_k\in\pi^{\underline{{\cal E}}}$. 
Again, if $p_j \not\in \underline{{\cal E}}$ for some $1 \leq j \leq n$ we are done, as the rule is satisfied by $\underline{{\cal E}}$. Thus, 
$p_1,\ldots,p_n \in\underline{{\cal E}}$. In other words, ${\sim\!p_1},\ldots,{\sim\!p_n}\not\in {\cal E}$ and ${\sim\!q_1},\ldots,{\sim\!q_m}\in {\cal E}$. 
Since ${\cal E}$ is a stable extension of $\ABF(\pi)$, this means that $\pi, {\cal  E} \vdash r_1\lor\ldots\lor r_k$. 
Suppose now for a contradiction that ${\sim\!r_i}\in{\cal E}$ for every $1\leqslant i\leqslant k$. Then by {\sf [Res]}, $\pi, {\cal E} \vdash r_i$
for every $1\leqslant i\leqslant k$ and thus ${\cal E}$ attack itself, which contradicts the fact that ${\cal E}$ is conflict-free. 
Consequently, there is at least one $1\leqslant i\leqslant k$ such that ${\sim\!r_i}\not\in{\cal E}$, thus $r_i\in \underline{{\cal E}}$,
which means that $\underline{{\cal E}}$ satisfies $p_1,\ldots,p_n \sim\!q_1,\ldots, \sim\!q_m \rightarrow r_1\lor \ldots\lor r_k$. 

To show the minimality of $\underline{{\cal E}}$, suppose that there is an $M \subsetneq \underline{{\cal E}}$ that is a model of 
$\pi^{ \underline{{\cal E}}}$. Let $p \in \underline{{\cal E}}\setminus M$. Since $p\in\underline{{\cal E}}$,  ${\sim\!p}\not \in {\cal E}$. 
Since ${\cal E}$ is stable, this means that $\pi, {\cal E} \vdash p$. By Corollary~\ref{corol:black}, any model of $\pi^{\underline{{\cal E}}}$ 
satisfies $p$, a contradiction to $p \not\in M$. \qed
\end{proof}

\section{Representation of ABA by DLP}
\label{sec:ABA2DLP}

The main body of literature on ABA frameworks is concentrated on languages that consist solely of formulas of the form $p_1,\ldots,p_n\rightarrow p$ 
(where $p,p_1,\ldots,p_n$ are atomic formulas). As noted previously, for such assumption-based frameworks (or at least when the frameworks are flat) it has been 
shown that there is a straightforward translation into normal logic programs that preserve equivalence for all the commonly studied argumentation 
semantics (see~\cite{caminada2017equivalence,CS18}).  To the best of our knowledge, the more complicated classes of ABA frameworks that are considered 
in this paper for characterizing disjunctive logic programming (and which  are based on a logic allowing to reason with disjunctive rules of the form 
$p_1,\ldots,p_n,{\sim\!q}_1,\ldots,{\sim\!q}_m\rightarrow r_1\lor\ldots\lor r_k$)
have not been investigated for other purposes other than the translation of DLPs. We thus do not see any motivation for investigating the reverse 
translation from these assumption-based frameworks into disjunctive logic programs. We do believe, however, that it is interesting to see if the more 
general class of assumption-based frameworks based on an arbitrary propositional logic (as defined and studied in~\cite{AH21,AH25,HA20}) can be translated 
in a class of logic programs, probably more general than disjunctive ones. 
This is a subject for a future work.

\section{Beyond Two-Valued Stable Semantics}
\label{sec:other-semantics}

So far we have discussed two-valued stable models for (disjunctive) logic programs and their correspondence to the stable extensions of the induced 
ABA frameworks. In this section we check whether similar correspondence may be established between other semantics of logic programs, such as 
three-valued stable semantics, and other types of argumentative extensions.

For switching to a 3-valued semantics we add to the two propositional constants ${\sf T}$ and ${\sf F}$, representing in the language ${\cal L}$ the 
Boolean values, a third elements, denoted ${\sf U}$, which intuitively represents uncertainty. These constants correspond, respectively, to the three
truth values of the underlying semantics, ${\sf t}$, ${\sf f}$, ${\sf u}$, representing truth, falsity and uncertainty. 
As in Kleene's semantics~\cite{Kl50}, these values may be arranged in two orders:
\begin{itemize}
    \item a total order $\leq_t$, representing difference in the amount of {\em truth\/} that each value represent, in which ${\sf f} <_t {\sf u} <_t {\sf t}$, and
    \item a partial order $\leq_i$, representing difference in the amount of {\em information\/} that each value depicts, in which ${\sf u} <_i {\sf f}$ and ${\sf u} <_i {\sf t}$. 
\end{itemize}
In what follows, we denote by $-$ the $\leq_t$-involution, namely $-{\sf f} = {\sf t}$, $-{\sf t} = {\sf f}$, and $-{\sf u} = {\sf u}$.

\medskip
Accordingly, we extend the notion of interpretations of a logic program $\pi$ from subsets of ${\cal A}(\pi)$ (as in Definition~\ref{def:models}) to {\em pairs\/} of 
subsets of ${\cal A}(\pi)$: 

\begin{definition}
{\rm A {\em three-valued interpretation\/} of a program $\pi$ is a pair $M = (x,y)$, where $x \subseteq {\cal A}(\pi)$ is the set of the atoms that 
are assigned the value ${\sf t}$ and $y \subseteq {\cal A}(\pi)$ is the set of atoms assigned a value in $\{{\sf t},{\sf u}\}$. 
} \end{definition}

Clearly, in every three-valued interpretation it holds that $x\subseteq y$, and the two-valued interpretations (onto  $\{{\sf t},{\sf f}\}$) of the previous sections 
may be viewed as three-valued interpretations in which $x = y$.

\medskip
Interpretations may be compared by two order relations, generalized from the order relations among the truth values:
\begin{enumerate}
     \item the \emph{truth order\/} $\leq_t$, where $(x_1,y_1) \leq_t (x_2,y_2)$ iff $x_1 \subseteq x_2$ and $y_1 \subseteq y_2$, and  
     \item the \emph{information order\/} $\leq_i$, where $(x_1,y_1) \leq_i (x_2,y_2)$ iff $x_1 \subseteq x_2$ and $y_2 \subseteq y_1$.
\end{enumerate}
The information order represents differences in the ``precisions'' of the interpretations. Thus, the components of higher values according to this order represent 
tighter evaluations. The truth order represents increased `positive' evaluations. Truth assignments to complex formulas are then recursively defined as in the 
next definition (see also~\cite{Pr91}):

\begin{definition}
\label{def:3-val-int}
{\rm The truth assignments of a 3-valued interpretation $(x,y)$ are defined as follows:
\begin{itemize}
\item $(x,y)(p)=
\begin{cases}
      {\sf t} & \text{ if } {p} \in x \text{ and } {p} \in y,  \\
      {\sf u} & \text{ if } {p} \not\in x \text{ and }{p} \in y,  \\
      {\sf f} & \text{ if } {p} \not\in x \text{ and } {p} \not\in y,  \\
\end{cases}$ \smallskip
\item $(x,y)(\sim\!\phi)=- (x,y)(\phi)$,  
\item $(x,y)(\psi \land \phi)=lub_{\leq_t}\{(x,y)(\phi),(x,y)(\psi)\}$, 
\item $(x,y)(\psi \lor \phi)= glb_{\leq_t}\{(x,y)(\phi),(x,y)(\psi)\}$. 
\end{itemize}
} \end{definition}

The next definition is the three-valued counterpart of Definitions~\ref{def:satisfiability} and~\ref{def:stable-model}.

\begin{definition}
\label{def:3-val-sem-dlp}
{\rm Let $\pi$ be a logic program and let $(x,y)$ be a 3-valued interpretation of $\pi$.
\begin{itemize}
    \item $(x,y)$ is  a \emph{(3--valued) model\/} of $\pi$, if for every $\phi \rightarrow \psi \in \pi$, $(x,y)(\phi) \leq_t (x,y)(\psi)$. 
    \item The {\em Gelfond-Lifschitz transformation\/}~\cite{gelfond1991classical} of a disjunctively normal program
            $\pi$ with respect to a 3-valued interpretation $(x,y)$, denoted $\pi^{(x,y)}$, is the positive program obtained by replacing in every rule in $\pi$ of the form
            $q_1, \ldots, q_m, {\sim\!r_1}, \ldots, {\sim\!r_k} \rightarrow p_1 \lor \ldots \lor p_n$, any negated literal ${\sim\!r_i}$ ($1\leq i\leq k$) by:
            ${\sf F}$ if $(x,y)(r_i)={\sf t}$, ${\sf T}$ if $(x,y)(r_i)={\sf f}$, and ${\sf U}$ if $(x,y)(r_i)={\sf u}$. That is, replacing ${\sim\!r_i}$ by the 
            propositional constant that corresponds to $(x,y)(\sim\!r_i)$.
    \item An interpretation $(x,y)$ is a {\em 3-valued stable model\/} of $\pi$~\cite{Pr91}, if it is a $\leq_t$-minimal model of $\pi^{(x,y)}$.
\end{itemize}    
} \end{definition}

There is a host of other semantics for both normal and disjunctive logic programs, including several semantics that refine the 3-valued stable models by, e.g., selecting only the $\leq_i$-{\em minimal\/} 
3-valued stable models (as is done for normal logic programs in the so-called \emph{3-valued well-founded models}), or taking the $\leq_i$-{\em maximal\/} 3-valued stable models. 
The latter, called {\em M-stable models\/} (or {\em L-stable models\/}), were introduced for DLPs by Sacca and Zaniolo~\cite{sacca1990stable,sacca1991partial,sacca1997deterministic}.  
For normal logic programs, M-stable models coincide with {\em regular models\/}, as defined by You and Yuan~\cite[Definition~8]{YY94}. 
However, for disjunctive logic programs, M-stable models do {\em not\/} coincide with regular models, since the latter are guaranteed to exist (see~\cite[Proposition 5.1]{YY94}) whereas the former are 
not (as 3-valued stable models in general might not exist for DLPs, see~\cite{Pr91}). A detailed overview of the relations between the above-mentioned semantics, as well as other semantics for DLPs, 
is given by Eiter, Leone and Sacc\'a in~\cite{eiter1997partial}.

Likewise, there are many proposals in the literature for extending the well-founded semantics from normal logic programs to disjunctive logic programs. We refer, e.g.,  
to \cite{alcantara2005well,brass1995disjunctive,knorr2007comparison,seipel1998alternating,WZ05} for some examples. However, there is no consensus about which are the most 
suitable ones, or even what should be the criteria for comparing them (see~\cite{alcantara2005well,WZ05}), including existence and uniqueness.  In what follows, for considering the correspondence 
with ABA frameworks, we restrict our attention to the 3-valued stable models, leaving the investigations of other semantics such as the ones discussed above to future work.
\begin{note}
{\rm As shown in~\cite{caminada2017equivalence,CS18},
when $\pi$ is a normal logic program, its 3-valued models that are defined above have some counterparts in terms of the induced assumption-based 
argumentation framework $\ABF_{{\sf Norm}}(\pi)$, described in Example~\ref{ex:translations}. To recall these results, we first need to extend 
the notions in Definition~\ref{def:semantics} with the following argumentative concepts:

\begin{definition} 
\label{def:Dung-semantics}
{\rm Let $\ABF=\tup{\mathfrak{L}, \Gamma, \Lambda,-}$ be an assumption-based framework and let $\Delta \subseteq \Lambda$. The set of all the
formulas in $\Lambda$ that are attacked by $\Delta$ is denoted by $\Delta^+$, and the set of all the formulas in $\Lambda$ that are defended by $\Delta$
(namely, the formulas whose attackers are in $\Delta^+$) is denoted ${\sf Def}(\Delta)$. We say that $\Delta$ is a {\em complete extension\/} of $\ABF$, 
if it is conflict-free (Definition~\ref{def:semantics}) and ${\sf Def}(\Delta) = \Delta$ (namely, $\Delta$ defends exactly its own elements). 
A $\subseteq$-maximally complete extension of $\ABF$ is called {\em preferred extension\/} of $\ABF$, and a $\subseteq$-minimally complete extension 
of $\ABF$ is called {\em grounded extension\/} of $\ABF$.\footnote{Note that, using the notations in Definition~\ref{def:Dung-semantics}, 
$\Delta \subseteq \Lambda$ is a stable extension of $\ABF$ according to Definition~\ref{def:semantics}, iff $\Delta \cap \Delta^+  = \emptyset$ (namely, 
$\Delta$ is conflict-free) and $\Delta \cup \Delta^+ = \Lambda$.}
} \end{definition}

In~\cite{caminada2017equivalence,CS18},
it was shown that the 3-valued stable models of a normal logic program $\pi$ correspond to the 
complete extensions of $\ABF_{{\sf Norm}}(\pi)$, the M-stable models of $\pi$ correspond to preferred extensions of $\ABF_{{\sf Norm}}(\pi)$, 
and the well-founded model of $\pi$ corresponds to the grounded extension of $\ABF_{{\sf Norm}}(\pi)$. Using the notations in Definition~\ref{def:main-notations}, 
these results can be expressed as follows: 
\begin{itemize}
    \item If $(x,y)$ is a 3-valued stable  
             model of a normal logic program $\pi$, then $\overline{y}$ is a complete 
             extension of $\ABF_{{\sf Norm}}(\pi)$.
    \item If ${\cal E}$ is a complete 
            extension of $\ABF_{{\sf Norm}}(\pi)$, where $\pi$ is a normal logic program, then $(\floor{{\cal E}^+},\underline{{\cal E}})$ is a 3-valued stable 
            model of $\pi$.
\end{itemize}
We note that \cite{caminada2017equivalence,CS18} also establish a connection between the well-founded models [respectively, the M-stable models] for normal logic programs 
and the grounded extensions [respectively, the preferred extensions] of assumption-based argumentation frameworks.\footnote{In~\cite{caminada2017equivalence,CS18} the M-stable
models are called regular models, but as Caminada and Schulz restrict the attention to normal logic programs, these are the same as M-stable models (as we have noted before). 
The comparison of preferred extensions with M-stable models of normal logic programs is justified by the similarity of their intuitions, namely minimization of undecidability.}
In other words, true atoms $p \in x$ on the LP-side correspond to attacked negations of these atoms ${\sim\!p} \in \overline{y}^+$ on the ABA-side, whereas 
false atoms $p \in {\cal A}(\pi) \setminus y$ correspond to accepted negations of these atoms ${\sim\!p} \in\overline{y}$ on the ABA-side. We illustrate this with an example:

\begin{example}
{\rm Consider the normal logic program $\pi_4=\{{\sim\!q} \rightarrow q, \ {\sim\!r} \rightarrow s\}$. This program has a single 3-valued stable model 
$(\{s\},\{s,q\})$, in which $s$ is true, $q$ is undecided and $r$ is false. The induced  assumption-based framework is
${\sf ABF}_{{\sf Norm}}(\pi_4)=\tup{\mathfrak{L}_{\sf MP},\pi_4, \{{\sim\!q}, {\sim\!r}, {\sim\!s}\},-}$. (A fragment of) its attack diagram is the following:

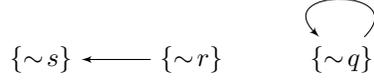
\begin{center}
\begin{tikzpicture}
\tikzset{edge/.style = {->,> = latex'}}

\node (simq) at (2,0) {$\{{\sim\!q}\}$};
\node (simr) at (0,0) {$\{{\sim\!r}\}$};
\node (sims) at (-2,0) {$\{{\sim\!s}\}$};

\draw[edge, loop, looseness=4] (simq) to (simq);
\draw[edge] (simr) to (sims);
\end{tikzpicture}
\end{center}

${\sf ABF}_{{\sf Norm}}(\pi_4)$ has a single preferred extension: $\{{\sim\!r}\}$, which attacks $\{\sim\!s\}$. Intuitively (also according to the 3-valued labeling 
semantics for this case, see~\cite{BCG18}), this makes ${\sim\!r}$ true (i.e., $r$ is false), ${\sim\!s}$ false (and so $s$ is true), and ${\sim\!q}$ undecided (since it is 
neither accepted nor attacked by ${\sim\!r}$).\footnote{In general, an extension ${\cal E} \subseteq \Lambda$ may be associated with a 3-valued `labeling' in which 
the element of ${\cal E}$ are accepted, the elements that are attacked by ${\cal E}$ (those in ${\cal E}^+$) are rejected, and the other elements are undecided 
(see, e.g.,~\cite{BCG18}). In the notations of Definition~\ref{def:3-val-int}, then, this corresponds to the 3-valued interpretation $({\cal E},\Lambda - {\cal E}^+)$.}
This corresponds to the 3-valued stable model of $\pi_4$, which is in-line with the results stated above. 
} \end{example}
} \end{note}

The correspondence described above, between 3-valued stable models of logic programs and extensions of the induced argumentation frameworks, 
does {\em not\/} carry on to disjunctive logic programs. This immediately follows from the fact that, while a 3-valued stable model (respectively: 
a $\leq_i$-minimal 3-valued stable model, a $\leq_i$-maximal 3-valued stable model) of a DLP are not guaranteed to exist~\cite{Pr91},\footnote{Indeed,
Przymusinski~\cite{Pr91} shows that the DLP  
$\pi = \{\rightarrow{\sim\!p} \vee {\sim\!q} \vee {\sim\!r}, \  {\sim\!r} \rightarrow p, \ {\sim\!p} \rightarrow q, \ {\sim\!q} \rightarrow r\}$ 
does not have a 3-valued stable model. There are other semantics that are guaranteed to exist for this program, such as the regular models \cite{YY94} or stationary 
semantics~\cite{przymusinski1991stationary}. We will see below that for these semantics, the correspondence breaks down as well.}
complete (respectively: grounded, preferred) extensions of the induced ABF always exist. Clearly, this consideration is
independent of the used translation. However, the next example shows that, moreover, {\em even when 3-valued stable models of a DLP do exist, the correspondence 
to the related ABF extensions ceases to hold under our translation\/}.

\begin{example}
{\rm Consider the disjunctive logic program $\pi_5 =\{{\sim\!q} \rightarrow q, \ q \rightarrow r \lor s\}$. The 3-valued stable models of $\pi_5$ are 
$(\emptyset, \{q,r\})$ and $(\emptyset, \{q,s\})$. 

The  assumption-based argumentation that is induced by $\pi_5$ is ${\sf ABF}(\pi_5)= \tup{\mathfrak{L},\pi_5, \{{\sim\!q},{\sim\!r},{\sim\!s}\},-}$.
This results in the following (fragment of the) attack diagram:

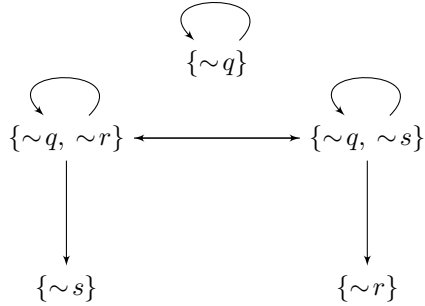
\begin{center}
\begin{tikzpicture}
\tikzset{edge/.style = {->,> = latex'}}

\node (simq) at (0,0) {$\{{\sim\!q} \}$};
\node (simqsimr) at (-2,-1) {$\{{\sim\!q},\:{\sim\!r}\}$};
\node (simqsims) at (2,-1) {$\{{\sim\!q},\:{\sim\!s}\}$};
\node (simr) at (2,-3) {$\{{\sim\! r}\}$};
\node (sims) at (-2,-3) {$\{{\sim\! s}\}$};

\draw[edge, loop, looseness=4] (simq) to (simq);
\draw[edge, loop, looseness=4] (simqsimr) to (simqsimr);
\draw[edge, loop, looseness=4] (simqsims) to (simqsims);

\draw[edge] (simqsimr) to (sims);
\draw[edge] (simqsimr) to (simqsims);

\draw[edge] (simqsims) to (simr);
\draw[edge] (simqsims) to (simqsimr);
\end{tikzpicture}
\end{center}

The unique complete set of assumptions in this case is $\emptyset$. This set does not attack any assumption, thus corresponding to $(\emptyset,\{q,r,s\})$. 
We thus see here that there is a discrepancy between the three-valued stable models of $\pi_5$ and complete extensions of ${\sf ABF}(\pi_5)$.
} \end{example}

\begin{note}
{\rm A closer inspection of the last example may reveal some of the reasons for the inadequacy of our translation under 3-valued semantics: In the 3-valued setting, 
the interpretations $(\emptyset,\{q,r\})$, $(\emptyset,\{q,s\})$, and $(\emptyset,\{q,r,s\})$, are all models of $\pi_5$, whereas only the first two are stable.
This means that in this setting, ${\sf f}$-assignments are not minimized (indeed, in either of the first two models, the set of the atoms that are assigned ${\sf f}$
properly contains the corresponding set of the other model). To reflect such a minimization in the induced argumentation framework, additional rules are required. 
For instance, in the last example, if $r\lor s$ is undecided (since ${\sim\!q}$ is undecided), one needs a rule that would bring about mutual attacks between 
${\sim\!r}$ and ${\sim\!s}$.
} \end{note}

We conclude this section by noting that some other semantics for disjunctive logic programs have been suggested in the literature, e.g.\ 
stationary semantics~\cite{Pr91}, (weakly) supported semantics~\cite{brass1995disjunctive,heyninck2022non}, determining inference 
semantics~\cite{shen2019determining}, semi-equilibrium semantics~\cite{amendola2016semi,heynincknon}, and several variants of the well-founded
semantics~\cite{alcantara2005well,heynincknon,knorr2007comparison,seipel1998alternating}.  Whether these semantics can be characterized by argumentation 
frameworks is a subject for future work.

\section{Extended Disjunctive Logic Programs}
\label{sec:EDLP}

Intuitively, the connective $\sim$ may be understood as representing `negation as failure' (to prove the converse)~\cite{Cl78}. This kind of negation is also known 
as `weak negation'. Extended (disjunctive) logic programs are obtained by introducing another negation connective, $\neg$, which acts as an explicit (strong, classical) 
negation, and allowing $\neg$-literals (namely, atomic formulas or their explicit negation) instead of the atoms in the rules~$(\star)$ of Definition~\ref{def:DLP}. Formally: 

\begin{definition}
\label{def:EDLP}
{\rm An {\em extended disjunctive logic program\/} $\pi$~\cite{Pr91} is a finite set of rules of the form
\[ (\star\star) \hspace*{15mm} l_1, \ldots, l_m, {\sim\!l_{m+1}}, \ldots, {\sim\!l_{m+k}}\rightarrow l_{m+k+1} \lor \ldots \lor l_{m+k+n}  \] 
where $m,k \geq 0$, $n \geq 1$, and each $l_i$ $(1 \leq i \leq m+k+n)$ is a {\em $\neg$-literal\/}, that is: $l_i \in \{p_i,\neg p_i\}$ for some $p_i \in {\cal A}(\pi)$.
} \end{definition}

The semantics of extended disjunctive logic programs is defined by reduction to disjunctive logic programs, using a standard method in logic programming that views 
a literal of the form $\neg p$ as a strangely written atomic formula. Under this view, rules of the form~$(\star\star)$ may be treated just as if they are of the form~$(\star)$. 
Relations between atomic formulas that correspond to positive and negative occurrences of the same atomic assertion (namely, without or with a leading explicit negation, 
respectively) are made on the semantic level. Indeed, while the semantics of disjunctive logic programs is two-valued (where an atom $p$ is verified in a model $M$ if $p \in M$ 
and is falsified otherwise), the semantics of extended disjunctive programs is {\em four-valued\/}. To see this, given a rule $r$ of the form~$(\star\star)$, replace every 
occurrence of $p$ by $p^+$ and every occurrence of $\neg p$ by $p^-$. Denote by $r^{\pm}$  by the resulting rule (and by $\pi^{\pm}$ the corresponding program).
Suppose that $M$ is a model of $r^{\pm}$. Then:
\begin{itemize}
     \item $p$ is {\em true\/} in $M$, if $p^+ \in M$ and $p^- \not\in M$,
     \item $p$ is {\em false\/} in $M$, if $p^+ \not\in M$ and $p^- \in M$,
     \item $p$ is {\em contradictory\/} in $M$, if $p^+ \in M$ and $p^- \in M$,
     \item $p$ is {\em undecided\/} in $M$, if $p^+ \not\in M$ and $p^- \not\in M$.
\end{itemize}

The next example illustrates this distinction between the semantics of disjunctive logic programs and extended disjunctive logic programs, and shows how this is reflected by 
the stable extensions of the corresponding ABFs.

\begin{example}
{\rm \
\begin{enumerate}
     \item Let $\pi_6 = \{ {\sim\!p} \rightarrow \neg p \}$. This program depicts a `closed world assumption'~\cite{Re78} regarding $p$:  Any statement that is true is 
             also known to be true, therefore, if $p$ is not  known to be true, it must be false. The conversion of $\pi_6$ to a normal logic program is 
             $\pi_6^{\pm} = \{ {\sim\!p^+} \rightarrow p^- \}$,
             the stable model is $\{p^-\}$, which in turn is associated with the the two-valued stable model of $\pi_6$, in which $p$ is false (see above). In the corresponding 
             assumption-based framework $\ABF(\pi_6) = \tup{{\mathfrak L},\pi_6^{\pm},\{{\sim\!p}^+,{\sim\!p}^-\},-}$ we have that $\{{\sim\!p}^+\}$ attacks 
             $\{{\sim\!p}^-\}$, and so $\{{\sim\!p}^+\}$ is the unique stable model in this case (as indeed indicated in Proposition~\ref{prop:model-to-extension}).            

             Note that if $\pi_6$ is extended with $\{p\}$, the situation is reversed: the sole stable extension of $(\pi_6 \cup \{p\})^{\pm}$ is $\{p^+\}$, and in the 
             corresponding ABF we have that  $\{{\sim\!p}^+\}$ is attacked by $\{{\sim\!p}^-\}$. Thus, the latter is the unique stable extension in this case, as expected.
    \item Consider the following extended disjunctive logic program: 
             $$\pi_7=\{\rightarrow \lnot p\lor \lnot q, \ {\sim\!\!\lnot p}\rightarrow p, \ {\sim\!\!\lnot q}\rightarrow q, \ p\rightarrow s, \ q\rightarrow s\}.$$ 
             We obtain the following translated disjunctive logic program:
             $$\pi_7^{\pm}=\{ \rightarrow  p^-\lor  q^-, \ {\sim\!p}^-\rightarrow p^+, \ {\sim\!q}^-\rightarrow q^+, \ p^+\rightarrow s^+, \ q^+\rightarrow s^+\}.$$ 
             The corresponding assumption-based framework is:
             \[\ABF(\pi_7^{\pm}) = \tup{{\mathfrak L},\pi_7^{\pm},\{{\sim\!p}^+,{\sim\!p}^-,{\sim\!q}^+,{\sim\!q}^-,{\sim\!s}^+,{\sim\!s}^-\},-}.\] 
             A fragment of the resulting attack diagram is given below:
     
\begin{center}
\begin{tikzpicture}
\tikzset{edge/.style = {->,> = latex'}}

\node (simq-) at (2,0) {$\{\sim\!q^- \}$};
\node (simp-) at (-2,0) {$\{\sim\!p^-\}$};
\node (simq+) at (2,2) {$\{\sim\!q^+ \}$};
\node (simp+) at (-2,2) {$\{\sim\!p^+\}$};
\node (sims+) at (0,3) {$\{\sim\!s^+ \}$};
\node (sims-) at (0,1.1) {$\{\sim\!s^- \}$};

\draw[edge] (simq-) to (simp-);
\draw[edge] (simp-) to (simq-);
\draw[edge] (simq-) to (simq+);
\draw[edge] (simp-) to (simp+);
\draw[edge] (simq-) to (sims+);
\draw[edge] (simp-) to (sims+);
\end{tikzpicture}
\end{center}

             There are two stable sets of assumptions, $\{\sim\!q^-, \sim\!p^+,\sim\!s^-\}$ and $\{\sim\!p^-, \sim\!q^+,\sim\!s^-\}$, corresponding (respectively) to the stable models 
             of $\pi_7^{\pm}$ $\{p^-,q^+,s^+\}$ and $\{q^-,p^+,s^+\}$, which in turn correspond (respectively) to the two stable models of $\pi_7$ $\{\lnot p, q, s\}$ and $\{\lnot q, p,s\}$.

    \item Let $\pi_8=\{\rightarrow p, \ \rightarrow q, \ \rightarrow r, \ \rightarrow \neg p \vee \neg q\}$. This logic program reflects an inconsistent information regarding
             $\{p,q\}$. The translated program is $\pi_8^{\pm}=\{\rightarrow p^+, \ \rightarrow q^+, \ \rightarrow r^+, \ \rightarrow p^- \vee q^-\}$, and the induce 
             induced ABFs is $\ABF(\pi_8^{\pm}) = \tup{{\mathfrak L},\pi_8^{\pm},\{{\sim\!p}^+,{\sim\!p}^-,{\sim\!q}^+,{\sim\!q}^-,{\sim\!r}^+\},-}$.
             Part of the attack diagram in this case is the following:
              
\begin{center}
\begin{tikzpicture}
\tikzset{edge/.style = {->,> = latex'}}

\node (empty) at (0,2) {$\{\}$};
\node (simp+) at (-2,0) {$\{\sim\!p^+\}$};
\node (simq+) at (0,0) {$\{\sim\!q^+\}$};
\node (simr+) at (2,0) {$\{\sim\!r^+\}$};
\node (simp-) at (-1,-1) {$\{\sim\!p^-\}$};
\node (simq-) at (1,-1) {$\{\sim\!q^-\}$};

\draw[edge] (simq-) to (simp-);
\draw[edge] (simp-) to (simq-);
\draw[edge] (empty) to (simp+);
\draw[edge] (empty) to (simq+);
\draw[edge] (empty) to (simr+);
\end{tikzpicture}
\end{center}

             The stable extensions of  $\ABF(\pi_8^{\pm})$ are therefore $\{{\sim\!q}^-\}$  and $\{{\sim\!p}^-\}$, corresponding to the stable models 
             $\{p^+,p^-,q^+,r^+\}$ and $\{q^+,q^-,p^+,r^+\}$ of $\pi_8^{\pm}$. In both models, inconsistency is `localized': $p$ is the only contradictory atoms 
             in the former model, and $q$ is the only contradictory atom in the latter.
\end{enumerate}
           
} \end {example}

We observe that in Item~3 of the last example the two stable models of the logic program $\pi_8^{\pm}$ coincide with the {\em paraconsistent stable models\/} 
of $\pi_8$ according to Sakama and  Inoue~\cite{SI95}. Next, we show that this is not a coincidence. First, we recall the definition of the paraconsistent stable 
semantics in~\cite{SI95}:

\begin{definition}
{\rm Let $M$ be a set of $\lnot$-literals.
\begin{itemize}
\item The satisfiability relation $\models$ is defined as follows (cf.\ Definition~\ref{def:models}, where $M$ is a set of atoms):
\begin{itemize}
\item $M \models l$ iff  $l\in M$ (for any $\lnot$-literal $l$)
\item $M \models l_1\lor\ldots\lor l_n$ iff $M\models l_i$ for some $1 \leq i \leq n$ (for any set of $\lnot$-literals $\{l_1,\ldots,l_n\}$),
\item $M \models  l_1, \ldots, l_m, {\sim\!l_{m+1}}, \ldots, {\sim\!l_{m+k}}\rightarrow l_{m+k+1} \lor \ldots \lor l_{m+k+n} $ iff whenever $M \models l_i$ for every 
         $1 \leq i \leq m$ and $M \not\models l_{m+i}$ for every $1 \leq i \leq k$, then $M \models l_{m+k+1} \lor \ldots \lor l_{m+k+n}$.
\end{itemize}
We say that $M$ is a {\em model\/} of an extended disjunctive logic program $\pi$ iff $M\models r$ for every $r\in \pi$. 

\item The (2-valued) {\em reducts\/} of extended DLPs can be constructed just as in the case without strong negation (cf.\ Definition~\ref{def:stable-model}), 
that is: $\pi^M$ consists of all the rules 
$ l_1, \ldots, l_m, \rightarrow l_{m+k+1} \lor \ldots \lor l_{m+k+n}$ such that 
$ l_1, \ldots, l_m, {\sim\!l_{m+1}}, \ldots, {\sim\!l_{m+k}}\rightarrow l_{m+k+1} \lor \ldots \lor l_{m+k+n} \in \pi$ and $l_{m+i}\not\in M$ for every $1 \leq i \leq k$.

\item A set of literals $M$ is a \emph{paraconsistent stable model\/} of $\pi$, iff $M$ is a $\subseteq$-minimal model of $\pi^M$. \footnote{These notions are 
identical to the corresponding notions of Sakama and  Inoue~\cite{SI95}, adapted to our notations.}
\end{itemize}
} \end{definition}

We now define a translation $\delta$ from $\pm$-atomic formulas to $\lnot$-literals by: $\delta(p^{-})=\lnot p$, $\delta(p^{+})= p$,  and denote
$\delta(\Theta)=\{\delta(p^\pm) \mid p^\pm \in\Theta\}$. 
Likewise, we define: 
$\delta^{-1}(\lnot p)=p^-$, $\delta^{-1}( p)=p^+$, and denote $\delta^{-1}(\Theta)=\{\delta^{-1}(l)\mid l\in\Theta\}$.

The following correspondence result is now easily obtained:

\begin{proposition}
let $\pi$ be an extended disjunctive logic program. Then:
\begin{itemize}
     \item [a)] If ${\cal E}$ is a stable extension of $\ABF(\pi^{\pm})$, then $\delta(\underline{{\cal E}})$ is a paraconsistent stable model of $\pi$. 
     \item [b)] If $M$ is a paraconsistent stable model of $\pi$, then $\overline{\delta^{-1}(M)}$ is a stable extension of $\ABF(\pi^{\pm})$.
\end{itemize}     
\end{proposition}

\begin{proof}
The proof immediately follows from the following two observations:
\begin{itemize}
\item A set $M$ of $\neg$-literals  is a paraconsistent stable model of $\pi$ iff $\delta(M)$ is a stable model of $\pi^\pm$.
        (This is straightforward from the definition of a paraconsistent stable model. In particular, there are no rules governing interactions between 
        $\lnot p$ and $p$ for any atom $p$, thus $p$ and $\lnot p$ behave as unrelated literals in this semantics.) 
\item If ${\cal E}$ is a stable extension of $\ABF(\pi^{\pm})$, then $\underline{{\cal E}}$ is a stable model of $\pi^{\pm}$, and vice-versa:
        if $M$ is a  stable model of $\pi^{\pm}$ then $\overline{M}$ is a stable extension of $\ABF(\pi^{\pm})$.
        (This follows from Propositions~\ref{prop:model-to-extension} and~\ref{prop:extension-to-model}.)  \qed
\end{itemize}
\end{proof}

By the last proposition it follows that our approach allows to capture the paraconsistent stable semantics for extended disjunctive logic programming 
introduced in~\cite{SI95} by a rather simple revision of the language. We refer also to~\cite{Wa22,Wa24}, where a further 
study of argumentative representations of extended disjunctive logic programming has been undertaken. The latter requires to extend the set of rules in the translation.
Thus, depending on the application at hand, each approach has its benefits and downsides.

\section{Related Work and Conclusion}
\label{sec:conclusion}

This work generalizes translations from logic programming into assumption-based argumentation to cover also (extended) disjunctive logic programs. 
Our framework was introduced in~\cite{HA19} and then generalized in~\cite{Wa22}, where 
Wakaki shows a semantic correspondence between generalized assumption-based argumentation 
(Definition~\ref{def:ABF}), based on the logic in Section~\ref{sec:translation}, and extended disjunctive logic programming.
In~\cite{Wa24} this correspondence is carried on to further formalisms for non-monotonic reasoning, including disjunctive default theory, parallel circumscription, 
and prioritized circumscription. Even though a different symbol for disjunction (denoted by $\mid\:$) is used in~\cite{Wa22,Wa24}, the semantics of this 
connective is the same as the one we consider here, when programs are restricted to $\neg$-free disjunctive logic programs (i.e., programs without strong negation, 
see also the appendix). When looking at more general classes of logic programs, different semantics have been proposed, e.g.\ Gelfond's answer set semantics for 
extended disjunctive logic programs (which are represented argumentatively in~\cite{Wa22,Wa24}). Thus, the family of extended logic programs considered here is 
somewhat different than the extended logic programs considered in~\cite{Wa22,Wa24}. In that respect, some difficulties that arise when representing disjunctive 
information in Reiter's default theory~\cite{reiter1980logic} are indicated and resolved in~\cite{Wa24}.
Rationality postulates and connections to related non-monotonic formalisms, such as  answer-set semantics and disjunctive default theories, are also discussed 
in~\cite{Wa22,Wa24}.

A work with a similar motivation is presented in~\cite{wang2000argumentation}, where a representation of DLPs by structured argumentation frameworks 
is proposed. In this framework, the assumptions are disjunctions of negated atoms ${\sim\!p}_1\lor \ldots \lor {\sim\!p}_n$, instead of just negated atoms 
as in our translation. Furthermore, unlike~\cite{wang2000argumentation}, we define our translation in assumption-based  argumentation, which means 
that meta-theoretical insights (e.g., complexity results~\cite{dimopoulos2002computational} or results on properties of the non-monotonic consequence
relations~\cite{AH25,vcyras2015non,HA20}), dialectical proof theories~\cite{du06dial,du06dial-2}, and different 
implementations~\cite{craven2013graph,toni2013generalised}, can be directly used. 

A representation of disjunctive logic programming by \emph{abstract argumentation\/} is studied in~\cite{bochman2003collective}. In that translation, 
nodes in the argumentation framework correspond to single assumptions ${\sim\!p}$, as opposed to \emph{sets\/} of such assumptions as in our translation. 
Because of this, the translation in~\cite{bochman2003collective} has to allow for attacks on \emph{sets of nodes\/}, instead of just nodes, 
necessitating a generalization of Dung's abstract argumentation frameworks \cite{Dung1995}. Since we work in assumption-based argumentation,
where nodes in the argumentation framework correspond to sets of assumptions, the argumentation frameworks generated by our translation are normal 
abstract argumentation frameworks. This is important since in that way results and implementations for abstract argumentation frameworks can be 
straightforwardly used and applied. Yet, the characterizations in~\cite{bochman2003collective} 
of 3-valued semantics of disjunctive logic programs by generalized abstract argumentation frameworks could provide valuable insights into the conditions 
(if any) under which a similar correspondence may be established with assumption-based argumentation.

Another related, but more distant line of work, is concerned with the integration of disjunctive reasoning in structured argumentation with defeasible rules
(see~\cite{BHS17,beirlaen2018critical}). We differ from this work both in the goal and the form of the knowledge bases.

In future work, we plan to generalize our results to other semantics for disjunctive logic programming, such as the disjunctive 
well-founded~\cite{brass1998characterizations}, extended well-founded~\cite{ross1992procedural}, and stationary semantics~\cite{przymusinski1991stationary}. 
Some of these semantics are based on ideas that are very similar to the ideas underlying several well-known argumentation semantics. 
Likewise, for example, both the stationary semantics for disjunctive logic programming and the preferred semantics from abstract
argumentation~\cite{Dung1995} can be characterized using three instead of two ``truth values''. Indeed, as noted in Section~\ref{sec:other-semantics},
for normal logic programs the correspondence between the 3-valued stable models for normal logic programs~\cite{Prz90} and complete labelings for 
ABA framework has been proven by Caminada and Schulz in~\cite{caminada2017equivalence,CS18}. On what conditions, if any, the correspondence holds also
for disjunctive logic programs, is still an open question. Finally, we hope to extend our results to more expressive languages, such as epistemic~\cite{gelfond1994logic}
and parametrized~\cite{gonccalves2010parametrized} logic programming.

\subsection*{Acknowledgments} 

We would like to thank the reviewers of this paper for their detailed and insightful comments. The work on this paper was partially 
supported by the Israel Science Foundation (Grant No. 550/19).

\bibliographystyle{plain}  
\bibliography{main}

\begin{thebibliography}{10}

\bibitem{alcantara2005well}
Joao Alc{\^a}ntara, Carlos~Viegas Dam{\'a}sio, and Lu{\'\i}s~Moniz Pereira.
\newblock A well-founded semantics with disjunction.
\newblock In {\em Proceedings of 21st International Conference on Logic
  Programming (ICLP'05)}, volume 3668 of {\em Lecture Notes in Computer
  Science}, pages 341--355. Springer, Springer, 2005.

\bibitem{amendola2016semi}
Giovanni Amendola, Thomas Eiter, Michael Fink, Nicola Leone, and Jo{\~a}o
  Moura.
\newblock Semi-equilibrium models for paracoherent answer set programs.
\newblock {\em Artificial Intelligence}, 234:219--271, 2016.

\bibitem{AH21}
Ofer Arieli and Jesse Heyninck.
\newblock Simple contrapositive assumption-based argumentation part {II:}
  {R}easoning with preferences.
\newblock {\em Journal of Approximate Reasoning}, 139:28--53, 2021.

\bibitem{AH25}
Ofer Arieli and Jesse Heyninck.
\newblock Simple contrapositive assumption-based argumentation frameworks with
  preferences: Partial orders and collective attacks.
\newblock {\em Journal of Approximate Reasoning}, 178:109340, 2025.

\bibitem{BCG18}
Pietro Baroni, Martin~W. Caminada, and Massimiliano Giacomin.
\newblock Abstract argumentation frameworks and their semantics.
\newblock In P.~Baroni, D.~Gabbay, M.~Giacomin, and L.~van~der Torre, editors,
  {\em Handbook of Formal Argumentation}, volume~1, pages 159--236. College
  Publications, 2018.

\bibitem{BHS17}
Mathieu Beirlaen, Jesse Heyninck, and Christian Stra{\ss}er.
\newblock Reasoning by cases in structured argumentation.
\newblock In {\em Proceedings of the 32nd ACM Symposium on Applied Computing
  (SAC'17)}, pages 989--994. ACM, 2017.

\bibitem{beirlaen2018critical}
Mathieu Beirlaen, Jesse Heyninck, and Christian Stra{\ss}er.
\newblock A critical assessment of {P}ollock's work on logic-based
  argumentation with suppositions.
\newblock In {\em Proceedings of 17th International Workshop on Non-Monotonic
  Reasoning (NMR'18)}, pages 63--72, 2018.

\bibitem{bochman2003collective}
Alexander Bochman.
\newblock Collective argumentation and disjunctive logic programming.
\newblock {\em Journal of logic and computation}, 13(3):405--428, 2003.

\bibitem{Bondarenko1997}
Andrei Bondarenko, Phan~Minh Dung, Robert Kowalski, and Francesca Toni.
\newblock An abstract, argumentation-theoretic approach to default reasoning.
\newblock {\em Artificial Intelligence}, 93(1):63--101, 1997.

\bibitem{brass1995disjunctive}
Stefan Brass and J{\"u}rgen Dix.
\newblock Disjunctive semantics based upon partial and bottom-up evaluation.
\newblock In {\em Proceedings of the Twelfth International Conference on Logic
  Programming (ICLP'95)}, pages 199--213. {MIT} Press, 1995.

\bibitem{brass1998characterizations}
Stefan Brass and J{\"u}rgen Dix.
\newblock Characterizations of the disjunctive well-founded semantics:
  confluent calculi and iterated {GCWA}.
\newblock {\em Journal of automated reasoning}, 20(1-2):143--165, 1998.

\bibitem{caminada2017equivalence}
Martin Caminada and Claudia Schulz.
\newblock On the equivalence between assumption-based argumentation and logic
  programming.
\newblock {\em Journal of Artificial Intelligence Research}, 60:779--825, 2017.

\bibitem{CS18}
Martin Caminada and Claudia Schulz.
\newblock On the equivalence between assumption-based argumentation and logic
  programming (extended abstract).
\newblock In {\em Proceedings of the 27trh International Joint Conference on
  Artificial Intelligence (IJCAI'2018)}, pages 5578--5582. ijcai.org, 2018.

\bibitem{Cl78}
Keith~L. Clark.
\newblock Negation as failure.
\newblock In M.~L. Ginsberg, editor, {\em Readings in Nonmonotonic Reasoning},
  pages 311--325. Morgan Kaufmann, 1978.

\bibitem{craven2013graph}
Robert Craven, Francesca Toni, and Matthew Williams.
\newblock Graph-based dispute derivations in assumption-based argumentation.
\newblock In {\em Peoceeings of the 2nd International Workshop on Theory and
  Applications of Formal Argumentation (TAFA'13)}, volume 8306 of {\em Lecture
  Notes in Computer Science}, pages 46--62. Springer, 2013.

\bibitem{CXST18}
Kristijonas {\v{C}}yras, Xiuy iFan, Claudia Schulz, and Francesca Toni.
\newblock Assumption-based argumentation: Disputes, explanations, preferences.
\newblock In P.~Baroni, D.~Gabbay, M.~Giacomin, and L.~van~der Torre, editors,
  {\em Handbook of Formal Argumentation}, volume~1, pages 365--408. College
  Publications, 2018.

\bibitem{vcyras2015non}
Kristijonas {\v{C}}yras and Francesca Toni.
\newblock Non-monotonic inference properties for assumption-based
  argumentation.
\newblock In {\em Proceedings of the 3rd International Workshop on Theory and
  Applications of Formal Argumentation}, volume 9524 of {\em Lecture Notes in
  Computer Science}, pages 92--111. Springer, 2015.

\bibitem{dimopoulos2002computational}
Yannis Dimopoulos, Bernhard Nebel, and Francesca Toni.
\newblock On the computational complexity of assumption-based argumentation for
  default reasoning.
\newblock {\em Artificial Intelligence}, 141(1-2):57--78, 2002.

\bibitem{Dung1995}
Phan~Minh Dung.
\newblock On the acceptability of arguments and its fundamental role in
  nonmonotonic reasoning, logic programming and n-person games.
\newblock {\em Artificial Intelligence}, 77:321--358, 1995.

\bibitem{du06dial}
Phan~Minh Dung, Robert~A Kowalski, and Francesca Toni.
\newblock Dialectic proof procedures for assumption-based, admissible
  argumentation.
\newblock {\em Artificial Intelligence}, 170(2):114--159, 2006.

\bibitem{du06dial-2}
Phan~Minh Dung, Paolo Mancarella, and Francesca Toni.
\newblock A dialectic procedure for sceptical, assumption-based argumentation.
\newblock In {\em Proceedings of COMMA'2006}, volume 144 of {\em Frontiers in
  Artificial Intelligence and Applications}, pages 145--156. {IOS} Press, 2006.

\bibitem{eiter1997disjunctive}
Thomas Eiter, Georg Gottlob, and Heikki Mannila.
\newblock Disjunctive datalog.
\newblock {\em ACM Transactions on Database Systems}, 22(3):364--418, 1997.

\bibitem{eiter1997partial}
Thomas Eiter, Nicola Leone, and Domenico Sac\'ca.
\newblock On the partial semantics for disjunctive deductive databases.
\newblock {\em Annals of Mathematics and Artificial Intelligence},
  19(1-2):59--96, 1997.

\bibitem{EC96}
David~W. Etherington and James~M. Crawford.
\newblock Toward efficient default reasoning.
\newblock In William~J. Clancey and Daniel~S. Weld, editors, {\em Proceedings
  of AAAI/IAAI'96}, pages 627--632. {AAAI} Press~/ {MIT} Press, 1996.

\bibitem{gelfond1994logic}
Michael Gelfond.
\newblock Logic programming and reasoning with incomplete information.
\newblock {\em Annals of mathematics and artificial intelligence},
  12(1-2):89--116, 1994.

\bibitem{GL88}
Michael Gelfond and Vladimir Lifschitz.
\newblock The stable model semantics for logic programming.
\newblock In {\em Proceedings of the 5th International Conference on Logic
  Programming (ICLP'88)}, pages 1070--1080. {MIT} Press, 1988.

\bibitem{gelfond1991classical}
Michael Gelfond and Vladimir Lifschitz.
\newblock Classical negation in logic programs and disjunctive databases.
\newblock {\em New generation computing}, 9(3-4):365--385, 1991.

\bibitem{GPLT91}
Michael Gelfond, Halina Przymusinska, Vladimir Lifschitz, and Miroslaw
  Truszczynski.
\newblock Disjunctive defaults.
\newblock In {\em Proceedings of the 2nd International Conference on Principles
  of Knowledge Representation and Reasoning (KR'91)}, pages 230--237. Morgan
  Kaufmann, 1991.

\bibitem{gonccalves2010parametrized}
Ricardo Gon{\c{c}}alves and Jos{\'e}~J{\'u}lio Alferes.
\newblock Parametrized logic programming.
\newblock In {\em European Workshop on Logics in Artificial Intelligence},
  pages 182--194. Springer, 2010.

\bibitem{gottlob1994complexity}
Georg Gottlob.
\newblock Complexity and expressive power of disjunctive logic programming
  (research overview).
\newblock In {\em Proceedings of the 1994 International Symposium on Logic
  programming}, pages 23--42. Mit Press, 1994.

\bibitem{HA19}
Jesse Heyninck and Ofer Arieli.
\newblock An argumentative characterization of disjunctive logic programming.
\newblock In {\em Proceedings of the 19th {EPIA} Conference on Artificial
  Intelligence, Part {II}}, volume 11805 of {\em Lecture Notes in Computer
  Science}, pages 526--538. Springer, 2019.

\bibitem{HA20}
Jesse Heyninck and Ofer Arieli.
\newblock Simple contrapositive assumption-based argumentation frameworks.
\newblock {\em Journal of Approximate Reasoning}, 121:103--124, 2020.

\bibitem{heyninck2022non}
Jesse Heyninck, Ofer Arieli, and Bart Bogaerts.
\newblock Non-deterministic approximation fixpoint theory and its application
  in disjunctive logic programming.
\newblock {\em Artificial Intelligence}, 331:104110, 2024.

\bibitem{heynincknon}
Jesse Heyninck and Bart Bogaerts.
\newblock Non-deterministic approximation operators: Ultimate operators,
  semi-equilibrium semantics, and aggregates.
\newblock {\em Theory and Practice of Logic Programming}, 23(4):632--647, 2023.

\bibitem{Kl50}
Stephen~C. Kleene.
\newblock {\em Introduction to Metamathematics}.
\newblock Van Nostrand, 1950.

\bibitem{knorr2007comparison}
Matthias Knorr and Pascal Hitzler.
\newblock A comparison of disjunctive well-founded semantics.
\newblock In {\em Foundations of Artificial Intelligence (FAInt)}, volume 277
  of {\em {CEUR} Workshop Proceedings}. CEUR-WS.org, 2007.

\bibitem{Li85}
Vladimir Lifschitz.
\newblock Computing circumscription.
\newblock In Aravind~K. Joshi, editor, {\em Proceedings of the 9th
  International Joint Conference on Artificial Intelligence (IJCAI'85)}, pages
  121--127. Morgan Kaufmann, 1985.

\bibitem{Ll87}
John~W. Lloyd.
\newblock {\em Foundations of Logic Programming}.
\newblock Springer, 1987.

\bibitem{przymusinski1991stationary}
Teodor Przymusinski.
\newblock Stationary semantics for disjunctive logic programs and deductive
  databases.
\newblock In {\em Proceedings of the 1990 North American Conference on Logic
  programming}, pages 40--59. MIT Press, 1990.

\bibitem{Prz90}
Teodor~C. Przymusinski.
\newblock The well-founded semantics coincides with the three-valued stable
  semantics.
\newblock {\em Fundamenta Informaticae}, 13(4):445--463, 1990.

\bibitem{Pr91}
Teodor~C. Przymusinski.
\newblock Stable semantics for disjunctive programs.
\newblock {\em New Generation Computing}, pages 401--424, 1991.

\bibitem{Re78}
Raymond Reiter.
\newblock On closed world data bases.
\newblock In Hervé Gallaire and Jack Minker, editors, {\em Logic and Data
  Bases}, pages 119--140. 1978.

\bibitem{reiter1980logic}
Raymond Reiter.
\newblock A logic for default reasoning.
\newblock {\em Artificial Intelligence}, 13(1):81--132, 1980.

\bibitem{ross1992procedural}
Kenneth~A Ross.
\newblock A procedural semantics for well-founded negation in logic programs.
\newblock {\em The Journal of Logic Programming}, 13(1):1--22, 1992.

\bibitem{sacca1990stable}
Domenico Sacca and Carlo Zaniolo.
\newblock Stable models and non-determinism in logic programs with negation.
\newblock In {\em Proceedings of the ninth ACM SIGACT-SIGMOD-SIGART symposium
  on Principles of database systems}, pages 205--217, 1990.

\bibitem{sacca1991partial}
Domenico Sacca and Carlo Zaniolo.
\newblock Partial models and three-valued models in logic programs with
  negation.
\newblock In {\em LPNMR}, volume~91, pages 87--101, 1991.

\bibitem{sacca1997deterministic}
Domenico Sacca and Carlo Zaniolo.
\newblock Deterministic and non-deterministic stable models.
\newblock {\em Journal of Logic and Computation}, 7(5):555--579, 1997.

\bibitem{SI95}
Chiaki Sakama and Katsumi Inoue.
\newblock Paraconsistent stable semantics for extended disjunctive programs.
\newblock {\em Journal of Logic and Computation}, 5(3):265--285, 1995.

\bibitem{SI00}
Chiaki Sakama and Katsumi Inoue.
\newblock Abductive logic programming and disjunctive logic programming: their
  relationship and transferability.
\newblock {\em Journal of Logic Programming}, 44(1-3):75--100, 2000.

\bibitem{schulz2015graphical}
Claudia Schulz.
\newblock Graphical representation of assumption-based argumentation.
\newblock In {\em Twenty-Ninth AAAI Conference on Artificial Intelligence},
  2015.

\bibitem{schulz2015characterising}
Claudia Schulz, Ken Satoh, and Francesca Toni.
\newblock Characterising and explaining inconsistency in logic programs.
\newblock In {\em International Conference on Logic Programming and
  Nonmonotonic Reasoning}, pages 467--479. Springer, 2015.

\bibitem{schulz_toni_2016}
Claudia Schulz and Francesca Toni.
\newblock Justifying answer sets using argumentation.
\newblock {\em Theory and Practice of Logic Programming}, 16(1):59–110, 2016.

\bibitem{seipel1998alternating}
Dietmar Seipel.
\newblock An alternating well-founded semantics for query answering in
  disjunctive databases.
\newblock In {\em Proceedings of FQAS'98}, pages 341--353. Springer, 1998.

\bibitem{shen2019determining}
Yi-Dong Shen and Thomas Eiter.
\newblock Determining inference semantics for disjunctive logic programs.
\newblock {\em Artificial Intelligence}, 277:103165, 2019.

\bibitem{su2015extensions}
Ezgi Su.
\newblock {\em Extensions of equilibrium logic by modal concepts}.
\newblock PhD thesis, IRIT-Institut de recherche en informatique de Toulouse,
  2015.

\bibitem{thevapalan2021establish}
Andre Thevapalan, Jesse Heyninck, and Gabriele Kern-Isberner.
\newblock Establish coherence in logic programs modelling expert knowledge via
  argumentation.
\newblock In {\em CAUSAL’21: Workshop on Causal Reasoning and Explanation in
  Logic Programming}, volume 1613 of {\em CEUR Workshop Proceedings}, 2021.

\bibitem{toni2013generalised}
Francesca Toni.
\newblock A generalised framework for dispute derivations in assumption-based
  argumentation.
\newblock {\em Artificial Intelligence}, 195:1--43, 2013.

\bibitem{Wa20}
Toshiko Wakaki.
\newblock Consistency in assumption-based argumentation.
\newblock In {\em Proceedings of the the 8th International Conference on
  Computational Models of Argument (COMMA'20)}, volume 326 of {\em Frontiers in
  Artificial Intelligence and Applications}, pages 371--382. {IOS} Press, 2020.

\bibitem{Wa22}
Toshiko Wakaki.
\newblock Assumption-based argumentation for extended disjunctive logic
  programming.
\newblock In {\em Proceedings of the 12th International Symposium on
  Foundations of Information and Knowledge Systems (FoIKS-2022)}, volume 13388
  of {\em Lecture Note in Computer Science}, pages 35--54. Springer, 2022.

\bibitem{Wa24}
Toshiko Wakaki.
\newblock Assumption-based argumentation for extended disjunctive logic
  programming and its relation to nonmonotonic reasoning.
\newblock {\em Journal of Argument and Computation}, 15(3):309--353, 2024.

\bibitem{wang2000argumentation}
Kewen Wang.
\newblock Argumentation-based abduction in disjunctive logic programming.
\newblock {\em The Journal of Logic Programming}, 45(1-3):105--141, 2000.

\bibitem{WZ05}
Kewen Wang and Lizhu Zhou.
\newblock Comparisons and computation of well-founded semantics for disjunctive
  logic programs.
\newblock {\em ACM Transactions on Computational Logic}, (2), 2005.

\bibitem{YY94}
Jia-Huai You and Li-Yan Yuan.
\newblock A three-valued semantics for deductive databases and logic programs.
\newblock {\em Journal of Computer and System Sciences}, (2), 1994.

\end{thebibliography}

\begin{appendix}
\label{appendix}

\section{A Note on the Relation Between $\vee$ and $\mid$ in $\neg$-free EDLPs}

In this appendix we show that, for normal disjunctive logic programs, the semantics of the connective $\mid$ as introduced by~\cite{gelfond1991classical} and 
considered by~\cite{Wa22,Wa24}, is the same as the disjunction $\vee$ that we consider here. This is  claimed in the conclusion of the main paper.

\begin{definition}
\label{def:GL-EDLP}
{\rm An {\em extended disjunctive logic program in the sense of~{\rm\cite{gelfond1991classical}}} (EDLP, in short) is a finite set $\pi$
of rules of the form
\[ l_1, \ldots, l_m, {\sim\!l_{m+1}}, \ldots, {\sim\!l_{m+k}} \rightarrow l_{m+k+1} \mid \ldots \mid l_{m+k+n}  \]  
where $m,k \geq 0$, $n \geq 1$, and each $l_i$ $(1 \leq i \leq m+k+n)$ is a {\em $\neg$-literal\/} (that is: $l_i \in \{p_i,\neg p_i\}$ for some $p_i \in {\cal A}(\pi)$).
A rule is {\em $\sim$-free\/} if $k=0$, and it is {\em $\neg$-free\/}, if every literal in the rule is an atom. An EDLP $\pi$ is $\sim$-free (respectively, $\neg$-free)
if every rule in it is $\sim$-free (respectively, $\neg$-free).
} \end{definition}

Note that the rules in an EDLP, as defined in~\cite{gelfond1991classical}, are similar to those in~$(\star\star)$ for extended disjunctive logic programs 
(Definition~\ref{def:EDLP}, following~\cite{Pr91}), except that the disjunctive connective $\vee$ used in the rule heads of~$(\star\star)$ is replaced here by the 
connective $\mid$. When restricted to $\neg$-free EDLPs, we obtain disjunctive logic programs (DLPs) as in Definition~\ref{def:DLP}, i.e., rules of
the form~$(\star)$ again with $\mid$ replacing $\vee$ in the rule heads. The reason for initially defining EDLPs in the sense of~\cite{gelfond1991classical} and 
subsequently reducing them to DLPs (also in the sense of~\cite{gelfond1991classical}) is that EDLPs are essential for the semantics developed in~\cite{Wa22,Wa24}, 
while DLPs provide the appropriate setting for establishing the correspondence between $\vee$ and $\mid$.

\medskip
In~\cite{Wa24}, semantics of an EDLP in the sense of~\cite{gelfond1991classical} is defined as follows:

\begin{definition}
\label{def:elp:answer:set}
{\rm A set $M$ of $\lnot$-literals is an \emph{answer set\/} of a $\sim$-free EDLP $\pi$, if it is a $\subseteq$-minimal set satisfying the following two conditions:
\begin{enumerate}
\item for every rule $l_1, \ldots, l_m \rightarrow l_{m+k+1} \mid \ldots \mid l_{m+k+n}$, if $l_i \in M$ for every $i=1 \ldots m$ then $l_j\in M$ for some $m+k+1 \leq  j \leq m+k+n$,  and
\item if $M$ contains a pair of literals $p$ and $\lnot p$ then $M$ contains any literal in the language of $\pi$.
\end{enumerate}
The \emph{reduct\/} $\pi^M$ is defined as 
\[ \pi^M =  \left \{
    \begin{array}{l} 
           l_1, \ldots, l_m \rightarrow l_{m+k+1} \mid \ldots \mid l_{m+k+n} \ \mbox{ where } \smallskip \\
           l_1, \ldots, l_m, {\sim\!l_{m+1}}, \ldots, {\sim\!l_{m+k}} \rightarrow l_{m+k+1} \mid \ldots \mid l_{m+k+n}  \in \pi  \mbox{ and }      
           \{{\sim\!l_{m+1}}, \ldots, {\sim\!l_{m+k}} \} \cap  M=\emptyset 
     \end{array} \right \}  \]
Finally, $M$ is an answer set of an EDLP $\pi$ if $M$ is an answer set of its reduct $\pi^{M}$.\footnote{In~\cite{Wa24} it is written that ``$M$ is an answer set of $\pi$ if $x$ is 
\emph{the} answer set of $\pi^M$\:'' (the emphasis is ours), but this is clearly mistaken as even $\sim$-free EDLPs admit more than one answer set.}
} \end{definition}

Given a $\neg$-free EDLP in the sense of~\cite{gelfond1991classical} $\pi$ (Defitinition~\ref{def:GL-EDLP}) we denote by ${\sf Convert}_{\:\mid\Rightarrow\vee}(\pi)$ the 
DLP that is obtained from $\pi$ by replacing all the connectives $\mid$ in the heads of the rules of $\pi$ by $\vee$. Conversely, for a DLP $\pi$ (Definition~\ref{def:DLP}) we denote 
by ${\sf Convert}_{\:\vee\Rightarrow\mid}(\pi)$ the $\neg$-free EDLP in the sense of~\cite{gelfond1991classical} that is obtained from $\pi$ by replacing all the connectives $\vee$ 
in the heads of the rules of $\pi$ by $\mid$.

\begin{proposition} 
$M$ is an answer set of an $\neg$-free EDLP $\pi$ iff $M$ is a stable model of ${\sf Convert}_{\:\mid\Rightarrow\vee}(\pi)$. Likewise,  $M$ is a stable model 
of a DLP $\pi$, iff $M$ is an answer set of ${\sf Convert}_{\:\vee\Rightarrow\mid}(\pi)$.
\end{proposition}

\begin{proof}
Observe first that no $\subseteq$-minimal model that satisfies Condition~1 of Definition~\ref{def:elp:answer:set} contains any negated literal $\lnot p$, 
and thus ``logical explosion'' (Condition 2 of Definition~\ref{def:elp:answer:set}) will never occur. The rest of the proof is immediate, as clearly the reduct according to
Definition~\ref{def:elp:answer:set} is identical to the Gelfond-Lifschitz reduct (Definition~\ref{def:stable-model}). \qed
\end{proof}

\end{appendix}

\end{document}